%% file: root.tex
\documentclass{article}
\usepackage{arxiv}

\usepackage{graphicx,amsbsy,amssymb,url}
\usepackage{amsmath}
\usepackage{thmtools}
\usepackage{booktabs}
\usepackage{comment,dsfont}
\usepackage{rotating}
\usepackage{subcaption}
\usepackage{epstopdf}
\usepackage{epsfig}
\usepackage{cite}
\usepackage{amsmath}
\usepackage{mathtools}
\usepackage[noend]{algpseudocode}
\usepackage{algorithm}
\usepackage{float}
\usepackage{tikz}
\usepackage{xcolor}
\usepackage[linkcolor=red]{hyperref}
\usetikzlibrary{positioning,plotmarks,calc}
\usepackage{pgfplots}
\pgfplotsset{compat=1.9}

\usepackage{flushend}
\newtheorem{remark}{Remark}
\newtheorem{definition}{Definition}
\newtheorem{theorem}{Theorem}
\newtheorem{corollary}{Corollary}

\newcommand{\SREC}{S-REC}

\def\x{{x}}

\def\S{{\mathcal{S}}}
\newenvironment{proof}{\paragraph{Proof:}}{\hfill$\square$}

\newcommand{\s}{0.060}

\include{definitions_small}

\begin{document}

\title{
Provably Convergent Algorithms for Solving \\ Inverse Problems Using Generative Models
}

\author{Viraj Shah, Rakib Hyder, M. Salman Asif, and Chinmay Hegde
\thanks{V. Shah was with Iowa State University. He is now with the ECE Department at the University of Illinois, Urbana-Champaign. (e-mail: vjshah3@illinois.edu).}
\thanks{C. Hegde was with Iowa State University. He is now with the Tandon School of Engineering at New York University. (e-mail: chinmay.h@nyu.edu).}
\thanks{R. Hyder and M. Asif are with the ECE Department at the University of California Riverside. (e-mail: sasif@ece.ucr.edu).}%
\thanks{This work was completed when VS and CH were at Iowa State University, and were supported in part by grants CAREER CCF-1750920/2005804 and CCF-1815101, a faculty fellowship from the Black and Veatch Foundation, and an equipment donation from the NVIDIA Corporation. RH and MA were supported in part by grants CAREER CCF-2046293 and ONR N00014-19-1-2264, and equipment donation from NVIDIA. Parts of this manuscript appeared in short conference papers~\cite{ganICASSP},~\cite{ganAllerton18}, and \cite{hyder2019alternating}.
}}

\maketitle

\begin{abstract}
The traditional approach of hand-crafting priors (such as sparsity) for solving inverse problems is slowly being replaced by the use of richer learned priors (such as those modeled by deep generative networks). In this work, we study the algorithmic aspects of such a learning-based approach from a theoretical perspective. For certain generative network architectures, we establish a simple non-convex algorithmic approach that (a) theoretically enjoys linear convergence guarantees for certain linear and nonlinear inverse problems, and (b) empirically improves upon conventional techniques such as back-propagation. We support our claims with the experimental results for solving various inverse problems. We also propose an extension of our approach that can handle model mismatch (i.e., situations where the generative network prior is not exactly applicable). Together, our contributions serve as building blocks towards a principled use of generative models in inverse problems with more complete algorithmic understanding.
\end{abstract}


\input{common/intro.tex}

\input{common/related_work.tex}

\input{common/algo_and_math.tex}

\input{common/figure2}
\input{common/experiment_results}
\input{common/rakib-icassp}

\section{DISCUSSION}\label{sec:conc}


Our contributions in this paper are primarily theoretical. We also explored the practical benefits of our approach in the context of inverse problems such as compressive sensing and phase retrieval. The algorithms proposed in this paper are generic and can be potentially used to solve other nonlinear inverse problems as well. 



We make several assumptions to enable our analysis. Some of them (for example, restricted strong convexity/smoothness; incoherence) are standard analysis tools and are common in the high-dimensional statistics and compressive sensing literature. However, in order to be applicable, they need to be verified for specific problems. A broader characterization of problems that do satisfy these assumptions will be of great interest.



%



\bibliographystyle{unsrt}
\bibliography{./bib/chinbiblio,./bib/vsbib,./bib/mrsbiblio}
\end{document}

%% file: definitions_small.tex
\newcommand{\R}{\ensuremath{\mathbb{R}}}

\newcommand{\eps}{\ensuremath{\varepsilon}}

\DeclareMathOperator*{\argmin}{arg\,min}

\DeclarePairedDelimiter{\parens}{\lparen}{\rparen}
\DeclarePairedDelimiter{\abs}{\lvert}{\rvert}
\DeclarePairedDelimiter{\norm}{\lVert}{\rVert}
\DeclarePairedDelimiter{\ceil}{\lceil}{\rceil}
\DeclarePairedDelimiter{\floor}{\lfloor}{\rfloor}

\makeatletter
\let\oldparens\parens
\def\parens{\@ifstar{\oldparens}{\oldparens*}}
\let\oldnorm\norm
\def\norm{\@ifstar{\oldnorm}{\oldnorm*}}
\let\oldceil\ceil
\def\ceil{\@ifstar{\oldceil}{\oldceil*}}
\let\oldfloor\floor
\def\floor{\@ifstar{\oldfloor}{\oldfloor*}}
\let\oldabs\abs
\def\abs{\@ifstar{\oldabs}{\oldabs*}}
\makeatother

\newcommand{\reals}{\mathbb{R}}

\newcommand{\vect}[1]{\mathbf{#1}}

\newcommand{\iprod}[2]{\left\langle #1, #2 \right\rangle}

\newcommand{\twonorm}[1]{\left\| {#1} \right\|_2}

\newcommand{\sign}[1]{\operatorname{sign}\left(#1\right)}

\newcommand{\distop}[2]{\mathrm{dist}\left(#1,#2\right)}

\newcommand{\rbrak}[1]{\left(#1\right)}

\newcommand{\cbrak}[1]{\left\{#1\right\}}

\newcommand{\y}{{y}}
\newcommand{\e}{{e}}
\newcommand{\z}{{z}}
\newcommand{\w}{{w}}

\newcommand{\xo}{{x^*}}

\newcommand{\ai}{{a}_i}

\newcommand{\p}{{p}}

\newcommand{\A}{{A}}

\newcommand{\subspaces}{{\mathcal{M}}}

%% file: common/intro.tex
\section{Introduction}

\subsection{Motivation}

Inverse problems arise in a diverse range of application domains including computational imaging, optics, astrophysics, and seismic geo-exploration. In each of these applications, there is a target signal or image (or some other quantity of interest) to be obtained; a device (or some other physical process) records measurements of the target; and the goal is to reconstruct an estimate of the signal from the observations. 

%
Let us suppose that $x^* \in \R^n$ denotes the signal of interest and $y = \mathcal{A}(x^*) \in \R^m$ denotes the observed measurements. The aim is to recover (an estimate of) the unknown signal $x^*$ given $y$ and $\mathcal{A}$. Based on the forward measurement operator $\mathcal{A}$, the inverse problem can be defined in two broad categories of linear and nonlinear problems. Many important problems in signal and image processing can be modeled with a \emph{linear} measurement operator $\mathcal{A}$; examples include \emph{compressive sensing}, the classical problem of \emph{super-resolution} or the problem of \emph{image inpainting}.  In case of nonlinear inverse problems, the operator $\mathcal{A}$ exhibits a nonlinearity; examples include \emph{phase retrieval}, \textit{blind deconvolution}, and \emph{de-quantization}.

When $m < n$, the inverse problem is ill-posed, and some kind of prior (or regularizer) is necessary to obtain a meaningful solution. A common technique used to solve ill-posed inverse problems is to seek the minimum of a constrained optimization problem:
\begin{align}
\widehat{x} &= \argmin~F(x),~~\label{eq:cop}\\
&\text{s.t.}~~~x \in \mathcal{S},\nonumber
\end{align}
where $F$ is an objective function that typically depends on $y$ and $\mathcal{A}$, and $\mathcal{S} \subseteq \R^n$ captures some sort of \emph{structure} that $x^*$ is assumed to obey. 

Sparsity is a common modeling assumption, particularly in signal and image processing applications, where $\mathcal{S}$ becomes a set of sparse vectors in some (known) basis representation. The popular framework of \emph{compressive sensing} studies the special case where the forward measurement operator $\mathcal{A}$ can be modeled as a linear operator that satisfies certain (restricted) stability properties; when this is the case, accurate estimation of $x^*$ can be performed, assuming that the signal $x^*$ is sufficiently sparse~\cite{candes2006compressive}. 

Parallel to the development of algorithms that leverage sparsity priors, the last decade has witnessed analogous approaches for other families of structural constraints. These include structured sparsity~\cite{modelcs,surveyEATCS}, unions-of-subspaces~\cite{MarcoCISS}, dictionary models~\cite{elad2006image,aharon2006rm}, total variation models~\cite{chan2006total}, analytical transforms~\cite{sairprasad}, among many others. 

Lately, there has been renewed interest in prior models that are parametrically defined in terms of a \emph{deep neural network}. We call these \emph{generative network} models. Specifically, we define 
\[
\mathcal{S} = \{x \in \R^n~|~x  = G(z),~z \in \R^k \}
\] 
where $z$ is a $k$-dimensional latent parameter vector and $G$ is parameterized by the weights and biases of a $d$-layer neural network. One way to obtain such a model is to train a generative adversarial network~\cite{goodfellow2014generative}. Generative models have found remarkable applications in image analysis~\cite{zhu2016generative,brock2016neural,chen2016infogan,zhao2016energy}, and a well-trained generative model closely captures the notion of a signal (or image) being `natural'~\cite{berthelot2017began}. Indeed, generative neural network learning algorithms have been successfully employed to solve inverse problems such as image super-resolution and inpainting~\cite{yeh2016semantic, ledig2016photo}. However, most of these approaches are heuristic and provable characterization of such algorithms are not readily available.






\subsection{Contributions}
 Our goal in this paper is to take some initial steps towards a principled use of generative  priors for inverse problems by a) proposing and analyzing the well known projected gradient descent (PGD) algorithm for solving~\eqref{eq:cop} for both linear and nonlinear inverse problems;  b) building a general theoretical framework for analyzing performance of such approaches from an {algorithmic} standpoint. Specifically, apart from providing algorithms to solve inverse problems using generative network models, we also wish to understand the {algorithmic} costs involved with such algorithms: how computationally challenging they are, whether they provably succeed, and how to make such models robust.

The starting point of our work is the seminal paper by \cite{bora2017compressed}, who study the benefits of using generative models in the context of compressive sensing. In this paper, the authors pose the estimated target as the solution to a non-convex optimization problem and establish upper bounds on the \emph{statistical} complexity of obtaining a ``good enough'' solution. Specifically, they prove that if the generative network is a mapping $G : \R^k \rightarrow \R^n$ simulated by a $d$-layer neural network with width $\leq n$ and with activation functions obeying certain properties, then $m = O(kd \log n)$ random observations are sufficient to obtain a good enough reconstruction estimate. However, they do not explicitly discuss an \emph{algorithm} to perform such non-convex optimization. Moreover, the authors do not study the \emph{algorithmic} costs of solving the optimization problem, and standard results in non-convex optimization are sufficient to only obtain sub-linear convergence rates. In this work, we make several advances towards understanding the convergence properties of gradient descent (and related algorithms).

First: we establish a projected gradient descent (PGD) algorithm with linear convergence rates for the compressive sensing setup identical to~\cite{bora2017compressed}, and demonstrate its empirical benefits over previous work. This constitutes \textbf{Contribution I} of this paper. 

Second: we generalize this to a much wider range of \emph{nonlinear} inverse problems. Using standard techniques, we propose a generic version of our PGD algorithm named $\eps-PGD$ for solving~\eqref{eq:cop} where $F$ is a smooth and strongly-convex objective function. We analyze this algorithm and prove its linear convergence under certain regimes of its smoothness and strong-convexity parameters. We also provide empirical results for solving nonlinear inverse problems. This forms \textbf{Contribution II} of this paper. 

Third: we address the challenging inverse problem of \emph{phase retrieval}~\cite{candes2015phase,demanet2014stable}. This is similar to the compressed sensing setup described above, except with the additional difficulty that only the magnitude of the observations are available. We prove that a natural variation of PGD coupled with an intermediate phase estimation step converges (locally) linearly to the true solution. This forms \textbf{Contribution III} of this paper.

Fourth: a drawback of \cite{bora2017compressed} (and our contribution I) is the inability to deal with targets that are outside the range of the generative network model. This is not merely an artifact of their analysis; generative networks are rigid in the sense that once they are learned, they are incapable of reproducing any target outside their range. (This is in contrast with other popular parametric models such as sparsity models that exhibit a ``graceful decay'' property in the sense that if the sparsity parameter $s$ is large enough, such models capture all possible points in the target space.) 

We address this gap, and propose an alternative algorithm using our general framework. We call it Myopic $\varepsilon$-PGD algorithm. It is novel, nonlinear extension of our previous work~\cite{spinisit,spinIT}. Under (fairly) standard assumptions, this algorithm also can be shown to demonstrate linear convergence. This constitutes \textbf{Contribution IV} of this paper. 

In summary: we complement the work of \cite{bora2017compressed} and \cite{sparsegen} by providing PGD based algorithms for solving inverse problems, and algorithmic upper bounds for the corresponding problems studied in those works. Together, our contributions serve as further building blocks towards an algorithmic theory of generative models in both linear and nonlinear inverse problems.

\subsection{Techniques}
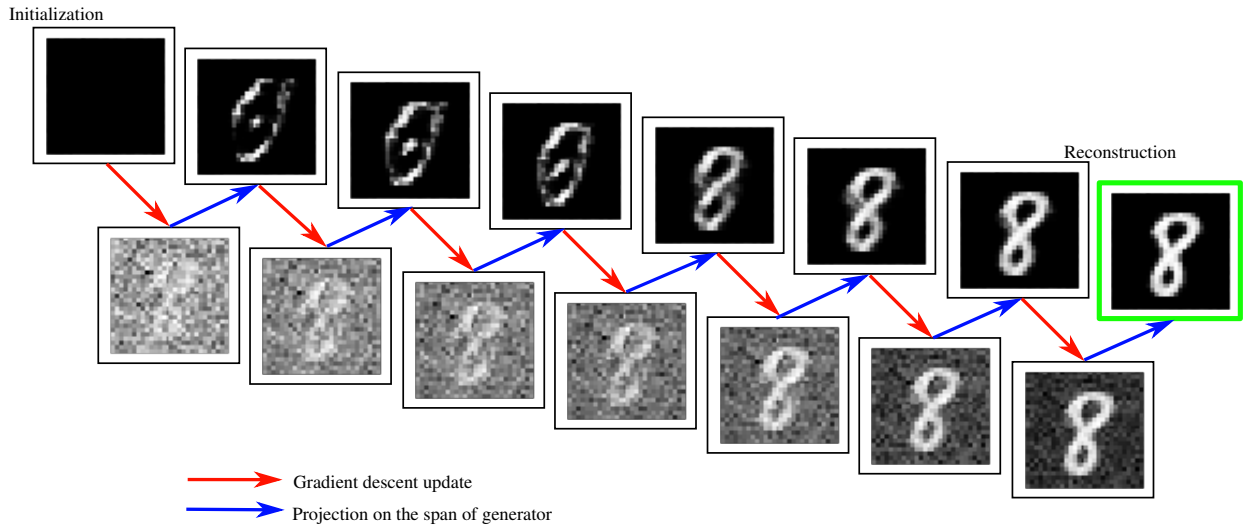
\begin{figure}
	\centering
	\def\svgwidth{\columnwidth}
	\input{./fig/image1.pdf_tex}
	\caption{\emph{Illustration of our framework. Starting from a zero vector, we perform a gradient descent update step (red arrow) and projection step (blue arrow) alternatively.}}
	\label{fig:intro1}
\end{figure}

At a high level, our algorithms are standard. The primary novelty is in their applications to generative network models, and some aspects of their theoretical analysis.

Suppose that $G : \R^k \rightarrow \R^n$ is the generative network model under consideration. The cornerstone of our analysis is the assumption of an $\varepsilon$-\emph{approximate} (Euclidean) projection oracle onto the range of $G$. We pre-suppose the availability of a computational routine $P_G$ that, given any vector $x \in \R^n$, can return a vector $x' \in \text{Range}(G)$ that approximately minimizes $\norm{x - x'}_2^2$.  The availability of this oracle, of course, depends on the nature of $G$. Some further comments on how to heuristically approximate this oracle are in Section~\ref{sec:conc}.

For a special case of linear inverse problems (and compressive sensing in particular), we assume such oracle to be simply a gradient descent routine minimizing the $\norm{x - x'}_2^2$ over the latent variable $z$ with $x'=G(z)$. Though this loss function is highly non-convex due to the presence of $G$, we find empirically that the gradient descent (implemented via back-propagation) works very well, and can be used as a projection oracle. Our procedure is depicted in Fig.~\ref{fig:intro1}. We choose a zero vector as our initial estimate ($x_0$), and in each iteration, we update our estimate by following the standard gradient descent update rule (red arrow in Fig.~\ref{fig:intro1}), followed by projection of the output onto the span of generator $(G)$ (blue arrow in Fig.~\ref{fig:intro1}). 

We support this specific PGD algorithm via a rigorous theoretical analysis. We show that the final estimate at the end of $T$ iterations is an approximate reconstruction of the original signal $x^*$, with very small reconstruction error; moreover, under certain sufficiency conditions on the linear operator $\mathcal{A}$, PGD demonstrates linear convergence, meaning that $T = \log(1/\delta)$ is sufficient to achieve $\delta$-accuracy. Further, we present a series of numerical results as  validation of our approach. 

We also provide a direct generalization of the above approach for nonlinear inverse problems, that we call $\eps-PGD$. We analyze this generic algorithm to show a linear convergence by assuming that the objective function in~\eqref{eq:cop} obeys the Restricted Strong Convexity/Smoothness assumptions~\cite{raskutti2010restricted}. With this assumption, the proof of convergence follows from a straightforward modification of the proof given in~\cite{jainkar2017}. Through our analysis, it indeed can be seen that the PGD algorithm for linear inverse problems is in fact a special case of $\eps$-PGD.

The fourth algorithm (Myopic $\varepsilon$-PGD) is a novel approach for handling model mismatch in the target. The main idea (following the lead of \cite{sparsegen}) is to pose the target $x^*$ as the superposition of two components: $x^* = G(z) + \nu$, where $\nu$ can be viewed as an ``innovation'' term that is $s$-sparse in some fixed, known basis $B$. The goal is now to recover both $G(z)$ and $\nu$. This is reminiscent of the problem of source separation or signal demixing~\cite{mccoyTropp2014}, and in our previous work~\cite{spinIT,NLDemix_TSP} we proposed greedy iterative algorithms for solving such demixing problems. We extend this work by proving a nonlinear extension, together with a new analysis, of the algorithm proposed in~\cite{spinIT}.

%% file: fig/image1.pdf_tex
\begingroup%
  \makeatletter%
  \providecommand\color[2][]{%
    \errmessage{(Inkscape) Color is used for the text in Inkscape, but the package 'color.sty' is not loaded}%
    \renewcommand\color[2][]{}%
  }%
  \providecommand\transparent[1]{%
    \errmessage{(Inkscape) Transparency is used (non-zero) for the text in Inkscape, but the package 'transparent.sty' is not loaded}%
    \renewcommand\transparent[1]{}%
  }%
  \providecommand\rotatebox[2]{#2}%
  \ifx\svgwidth\undefined%
    \setlength{\unitlength}{582.04980469bp}%
    \ifx\svgscale\undefined%
      \relax%
    \else%
      \setlength{\unitlength}{\unitlength * \real{\svgscale}}%
    \fi%
  \else%
    \setlength{\unitlength}{\svgwidth}%
  \fi%
  \global\let\svgwidth\undefined%
  \global\let\svgscale\undefined%
  \makeatother%
  \begin{picture}(1,0.42136317)%
    \put(0,0){\includegraphics[width=\unitlength,page=1]{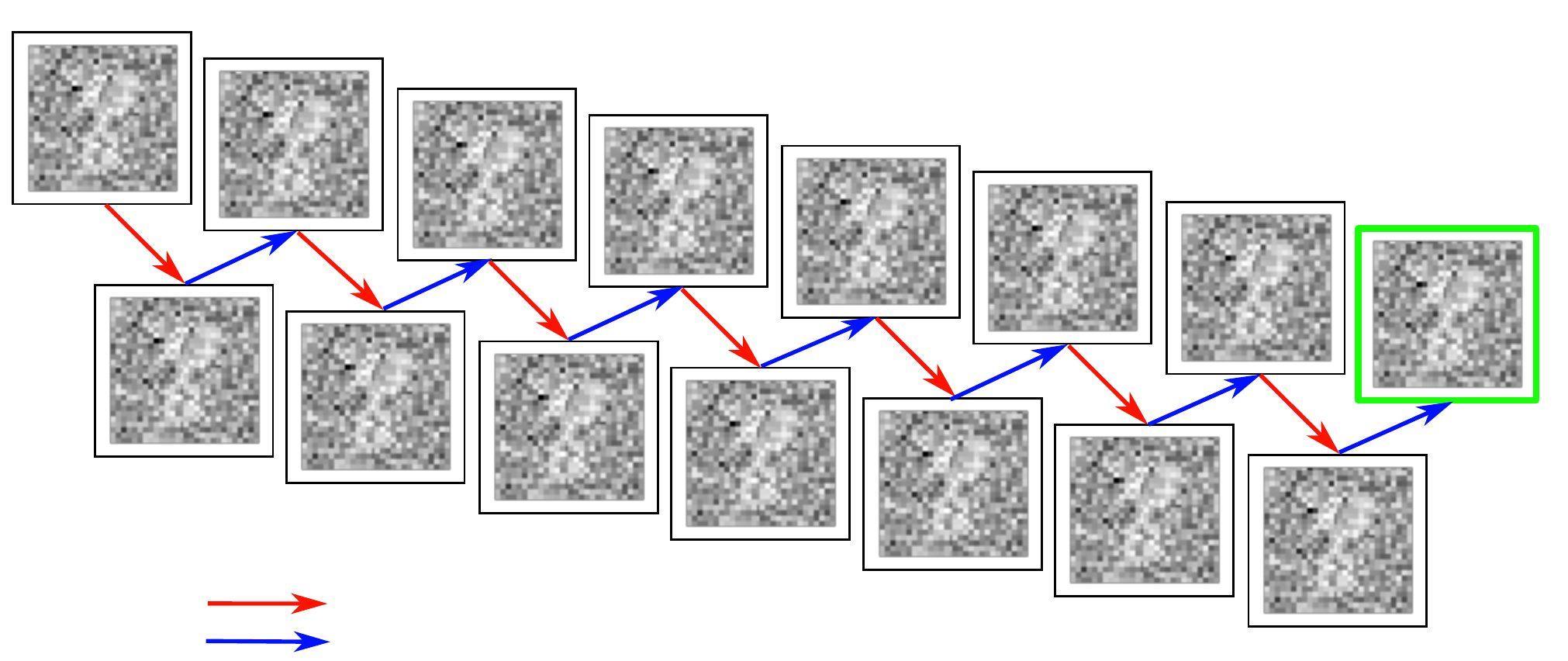}}%
    \put(0.21780872,0.02964805){\color[rgb]{0,0,0}\makebox(0,0)[lb]{\smash{\scriptsize{Gradient descent update}}}}%
    \put(0.21689982,0.00393108){\color[rgb]{0,0,0}\makebox(0,0)[lb]{\smash{\scriptsize{Projection on the span of generator}}}}%
    \put(0,0){\includegraphics[width=\unitlength,page=2]{./fig/image.pdf}}%
    \put(0.8384616,0.2955721){\color[rgb]{0,0,0}\makebox(0,0)[lb]{\smash{\scriptsize{Reconstruction}}}}%
    \put(-0.01125109,0.40700457){\color[rgb]{0,0,0}\makebox(0,0)[lb]{\smash{\scriptsize{Initialization}}}}%
  \end{picture}%
\endgroup%

%% file: common/related_work.tex
\section{Background and Related Work}

\subsection{Inverse problems}

The study of solving inverse problems has a long history. As discussed above, the general approach to solve an ill-posed inverse problem is to assumes that the target signal/image obeys a \emph{prior}. Classical methods mainly used hand-crafted signal priors to distinguish `natural' signals from the infinite set of feasible solutions. The prior can be encoded in the form of either a constraint set (as in Eq.~\eqref{eq:cop}) or an extra regularization penalty. Several methods (including \cite{donoho1995noising, xu2010image, dong2011image}) employ sparsity priors to solve \emph{linear} inverse problems such as denoising, super-resolution, and inpainting. 
Despite their successful practical and theoretical results, all such hand-designed priors often fail to restrict the solution space only to natural images, and it is easily possible to generate signals satisfying the prior but do not resemble natural data.

\subsection{Neural network models}

The last few years have witnessed the emergence of trained \emph{neural networks} for solving such problems. The main idea is to eschew hand-crafting any priors, and instead \emph{learn} an end-to-end mapping from the measurement space to the image space. This mapping is simulated via a deep neural network, whose weights are learned from a large dataset of input-output training examples  \cite{lecun2015deep}. The works \cite{kulkarni2016reconnet,mousavi2015deep,mousavi2017learning,xu2014deep, dong2016image,kim2016accurate,yeh2017semantic,metzler2020optica} have used this approach to solve several types of inverse problems, and has met with considerable success. However, the major limitations are that a new network has to be trained for each new linear inverse problem; moreover, most of these methods lack concrete theoretical guarantees. An exception of this line of work is the powerful framework of~\cite{rick2017one}, which does \emph{not} require retraining for each new problem; however, this too is not accompanied by theoretical analysis of statistical and computational costs.

\subsection{Generative networks}

A special class of neural networks that attempt to directly model the distribution of the input training samples are known as generative  adversarial training networks, or GANs \cite{goodfellow2014generative}. GANs have been shown to provide visually striking results \cite{arjovsky2017wasserstein,pix2pix,berthelot2017began,simonyan2018}. The use of GANs to solve linear inverse problems was advocated in \cite{bora2017compressed}. Specifically, given (noisy) linear observations $y = Ax^* + e$ of a signal $x^* \in \R^n$, assuming that $x^*$ belongs to the range of a generative network $G : \R^n \rightarrow \R^k$, this approach constructs the reconstructed estimate $\hat{x}$ as follows:
\[
\hat{z} = \arg \min_{z \in \R^k} \norm{y - A G(z)}_2^2,~~\hat{x} = G(\hat{z})
\]
If the observation matrix $A \in \R^{m \times n}$ comprises 
$
m = O(kd \log n)
$
i.i.d. Gaussian measurements, then together with regularity assumptions on the generative network, they prove that the solution $\hat{x}$ satisfies:
$$\norm{x^* - \hat{x}}_2 \leq C \norm{e}_2.$$
for some constant $C$ that can be reliably upper-bounded. In particular, in the absence of noise the recovery of $x^*$ is exact. However, there is no discussion of how computationally expensive this procedure is. Observe that the above minimization is highly non-convex (since for any reasonable neural network, $G$ is a non-convex function) and possibly also non-smooth (if the neural network contains non-smooth activation functions, such as rectified linear units, or ReLUs). More recently,~\cite{deepprior} improve upon the approach in~\cite{bora2017compressed} for solving more general nonlinear inverse problems (in particular, any inverse problem that has a computable derivative). Also, ~\cite{hand2019global} have analyzed the above problem for the \emph{untrained} case where the weights of the generative model $G$ obeys certain randomness assumptions. See, also,~\cite{jagatap2019algorithmic}.

Under similar statistical assumptions as~\cite{bora2017compressed}, the work of~\cite{ganICASSP} provably establishes a linear convergence rate, provided that a projection oracle (on to the range of $G$) is available, but only for the special case of compressive sensing. Our generalized result (Contribution II) extends this algorithm (and analysis) to more general nonlinear inverse problems.
 
More recently, \cite{raj2019gan} proposed a method that learns the a network-based projector for use in the PGD algorithm, making the projection step faster computationally. However, their theoretical result assumes the learned projector to be $\delta-$approximate, indicating that the effective training of the projector is crucial for the success of their method posing an additional challenge.

\subsection{Phase retrieval}
The phase retrieval problem has been extensively studied over the last few decades \cite{gerchberg1972phase,fienup1982phase,candes2013phaselift} and it appears in several applications, including optical imaging \cite{fienup1982phase,holloway2016toward}, microscopy \cite{tian2014multiplexed}, and X-ray crystallography \cite{miao1999extending}. 
Phase retrieval is a non-convex problem and classical solution methods rely on alternating projection heuristics; examples include Gerchberg-Saxton  \cite{gerchberg1972phase} and Fienup \cite{fienup1982phase}. In recent years, lifting-based methods were introduced that reformulate phase retrieval as a semidefinite program.  \cite{candes2013phaselift}. Subsequently, non-convex methods have been proposed for solving phase retrieval problem with theoretical performance guarantees  \cite{bahmani,goldstein2018phasemax,netrapalli,wang2017solving,candes2015codeddiff,copram,structphase}. Most of the non-convex methods rely on estimating a good initial solution via the so-called \textit{spectral initialization} method. A number of methods for solving phase retrieval using trained neural networks have been recently proposed \cite{metzler2020optica,barbastathis2019use,sinha2017lensless}.

\subsection{Model mismatch}

A limitation of most generative network models is that they can only reproduce estimates that are within their range; adding more observations or tweaking algorithmic parameters are completely ineffective if a generative network model is presented with a target that is far away from the range of the model. To resolve this type of model mismatch, the authors of \cite{sparsegen} propose to model the signal $x^*$ as the superposition of two components: a ``base'' signal $u = G(z)$, and an ``innovation'' signal $v = B \nu$, where $B$ is a known ortho-basis and $\nu$ is an $l$-sparse vector.  In the context of compressive sensing, the authors of \cite{sparsegen} solve a sparsity-regularized loss minimization problem:
\[
(\hat{z}, \hat{v}) = \arg \min_{z, v} \norm{B^T v}_1 + \lambda \norm{y - A(G(z) + v)}_2^2 .
\]
and prove that the reconstructed estimate $\hat{x} = G(\hat{z}) + \hat{v}$ is close enough to $x$ provided $m = O((k + l)d \log n)$ measurements are sufficient. However, as before, the algorithmic costs of solving the above problem are not discussed. Our third main result (Contribution III) proposes a new algorithm for dealing with model mismatches in generative network modeling, together with an analysis of its convergence and iteration complexity.

%% file: common/algo_and_math.tex
\section{Main Algorithms and Analysis}

Let us first establish some notational conventions. Below, $\norm{\cdot}$ will denote the Euclidean norm unless explicitly specified. We use $O(\cdot)$-notation in several places in order to avoid duplication of constants.  We use $F(\cdot)$ to denote a (scalar) objective function.

 \input{common/math_model.tex}
 \input{common/theory_results.tex}

\subsection{Contribution II: Solving nonlinear inverse problems}
\label{sec:nonlinear}
We now present the generic version of PGD algorithm suitable for large class of nonlinear inverse problems by generalizing for the loss function $F(\cdot)$ and the projection oracle $P_G$. 

We denote the $F(\cdot)$ to be a scalar function with continuous gradient, and assume that $F$ has a continuous gradient $\nabla F = \left(\frac{\partial F}{\partial x_i}\right)_{i=1}^n$ which can be evaluated at any point $x \in \R^n$. 

Recall that we wish to solve the problem
\begin{align}
\widehat{x} &= \argmin~F(x),~~\label{eq:cop0}\\
&\text{s. t.}~~~x \in \text{Range}(G),\nonumber
\end{align}
where $G$ is a generative network. To do so, we now employ a generalized version of \emph{projected gradient descent} algorithm using the $\varepsilon$-approximate projection oracle for $G$. The algorithm is described in Alg. \ref{alg:PGD}.

We define the $\varepsilon$-approximate projection oracle $P_G$ as, 

\begin{definition}[Approximate projection] A function $P_G : \R^n \rightarrow \text{Range}(G)$ is an $\varepsilon$-approximate projection oracle if for all $x \in \R^n$, $P_G(x)$ obeys:
\[
\norm{x - P_G(x)}^2_2 \leq \min_{z \in \mathbb{R}^k} \norm{x - G(z)}^2_2 + \varepsilon .
\]

We will assume that for any given generative network $G$ of interest, such a function $P_G$ exists and is computationally tractable\footnote{This may be a very strong assumption, but at the moment we do not know how to relax this in the very general case. Indeed, the computational complexity of our proposed algorithms is proportional to the complexity of such a projection oracle. For preliminary advances, see~\cite{lei2019inverting}}. Here, $\varepsilon > 0$ is a parameter that is known \emph{a priori}. 
\end{definition}

In contrast to our previous analysis, here we introduce more general restriction conditions on the $F(\cdot)$:
\begin{definition}[Restricted Strong Convexity/Smoothness]
Assume that $F$ satisfies $\forall x, y \in S$:
\[
\frac{\alpha}{2} \norm{x - y}_2^2 \leq F(y) - F(x) - \langle \nabla F(x), y - x \rangle \leq \frac{\beta}{2} \norm{x - y}_2^2 .
\]
for positive constants $\alpha, \beta$.
\end{definition}
\medskip
This assumption is by now standard; see~\cite{raskutti2010restricted,jainkar2017} for in-depth discussions. This means that the objective function is strongly convex / strongly smooth along certain directions in the parameter space (in particular, those restricted to the set $S$ of interest). The parameter $\alpha > 0$ is called the restricted strong convexity (RSC) constant, while the parameter $\beta > 0$ is called the restricted strong smoothness (RSS) constant. Clearly, $\beta \geq \alpha$. In fact, throughout the paper, we assume that 
$
1 \leq \frac{\beta}{\alpha} < 2,
$
which is a fairly stringent assumption but again, one that we do not know at the moment how to relax.

\begin{definition}[Incoherence]
A basis $B$ and $\text{Range}(G)$ are called $\mu$-incoherent if for all $u, u' \in \text{Range}(G)$ and all $v, v' \in \text{Span}(B)$, we have:
\[
| \langle u - u', v - v' \rangle | \leq \mu \norm{u - u'}_2 \norm{v - v'}_2.
\]
for some parameter $0 < \mu< 1$.
\end{definition}

\begin{remark}
In addition to the above, we will make the following assumptions in order to aid the analysis. Below, $\gamma$ and $\Delta$ are positive constants.
\begin{itemize}
\item $\norm{\nabla F(x^*)}_2 \leq \gamma$.
\item $\text{diam}(\text{Range}(G)) = \Delta$.
\item $\gamma \Delta \leq O(\varepsilon)$. 
\end{itemize}
Some comments about these assumptions may be warranted. The first says that the gradient of the loss at \emph{at the minimizer} is small. The second says that the range of $G$ is compact. The third links the parameters from the first two assumptions.
\end{remark}

\begin{algorithm}[t]
	\caption{$\varepsilon$-\textsc{PGD}}
	\label{alg:PGD}
	\begin{algorithmic}[1]
	\State \textbf{Inputs:} $y$, $T$, $\nabla$; \textbf{Output:}  $\widehat{x}$
	\State $x_0 \leftarrow \textbf{0}$ \hspace{17.8em} 
	\While {$t < T$}
	\State $w_t \leftarrow x_t - \eta \nabla F(x_t)$ \hspace{8.8em} 
	\State $x_{t+1} \leftarrow P_G(w_t) $ \hspace{0.6em} 
	\State $t \leftarrow t+1$
	\EndWhile
	\State $\widehat{x} \leftarrow x_{T}$
	\end{algorithmic}
\end{algorithm}

  We obtain the following theoretical result:

\begin{theorem}
\label{thm:pgd}
If $F$ satisfies RSC/RSS over $\text{Range}(G)$ with constants $\alpha$ and $\beta$, then $\varepsilon$-PGD (Alg.\ \ref{alg:PGD}) convergences linearly up to a ball of radius $O(\gamma \Delta) \approx O(\varepsilon)$. 
\[
F(x_{t+1}) - F(x^*) \leq \left( \frac{\beta}{\alpha} - 1\right) (F(x_t) - F(x^*)) + O(\varepsilon) \, .
\]
\end{theorem}
\begin{proof}
The proof is a minor modification of that in \cite{jainkar2017}. For simplicity we will assume that $\norm{\cdot}$ refers to the Euclidean norm. Let us suppose that the step size $\eta = \frac{1}{\beta}$. Define
\[
w_t = x_t - \eta \nabla F(x_t) .
\]
By invoking RSS, we get:
\begin{align*}
& F(x_{t+1}) - F(x_t) \\
& \leq \langle \nabla F(x_t) , x_{t+1} - x_t \rangle + \frac{\beta}{2} \norm{x_{t+1} - x_t}^2 \\
&=  \frac{1}{\eta} \langle x_t - w_t, x_{t+1} - x_t \rangle + \frac{\beta}{2} \norm{x_{t+1} - x_t}^2 \\
&= \frac{\beta}{2} \left( \norm{x_{t+1} - x_t}^2 + 2 \langle x_t - w_t,  x_{t+1} - x_t \rangle + \norm{x_t - w_t}^2 \right) \\
&~~~- \frac{\beta}{2} \norm{x_t - w_t}^2 \\
&= \frac{\beta}{2} \left(\norm{x_{t+1} - w_t}^2 - \norm{x_t - w_t}^2 \right),
\end{align*}
where the last few steps are consequences of straightforward algebraic manipulation. 

Now, since $x_{t+1}$ is an $\varepsilon$-approximate projection of $w_t$ onto $\text{Range}(G)$ and $x^* \in \text{Range}(G)$, we have:
\[
\norm{x_{t+1} - w_t}^2 \leq \norm{x^* - w_t}^2 + \varepsilon.
\]
Therefore, we get:
\begin{align*}
& F(x_{t+1}) - F(x_t) \\
&\leq \frac{\beta}{2} \left(\norm{x^* - w_t}^2 - \norm{x_t - w_t}^2 \right) + \frac{\beta \varepsilon}{2} \\
&=\frac{\beta}{2} \left(\norm{x^* - x_t + \eta \nabla F(x_t)}^2 - \norm{\eta \nabla F(x_t)}^2 \right) + \frac{\beta \varepsilon}{2} \\
&= \frac{\beta}{2} \left( \norm{x^* - x_t}^2 + 2 \eta \langle x^* - x_t, \nabla F(x_t) \rangle \right) + \frac{\beta \varepsilon}{2} \\
&= \frac{\beta}{2} \norm{x^* - x_t}^2 + \langle x^* - x_t, \nabla F(x_t) \rangle + \frac{\beta \varepsilon}{2}.
\end{align*}
However, due to RSC, we have:
\begin{align*}
\frac{\alpha}{2} \norm{x^* - x_t}^2 &\leq F(x^*) - F(x_t) - \langle x^* - x_t, \nabla F(x_t) \rangle, \\
\langle x^* - x_t, \nabla F(x_t) \rangle &\leq F(x^*) - F(x_t) - \frac{\alpha}{2} \norm{x^* - x_t}^2 .
\end{align*}
Therefore,
\begin{align*}
& F(x_{t+1}) - F(x_t) \\
&\leq \frac{\beta - \alpha}{2} \norm{x^* - x_t}^2 + F(x^*) - F(x_t) + \frac{\beta \varepsilon}{2} \\
&\leq \frac{\beta - \alpha}{2} \cdot \frac{2}{\alpha} \left( F(x_t) - F(x^*) - \langle x_t - x^*, \nabla F(x^*) \rangle \right)  \\
& ~~~~+ F(x^*) - F(x_t) + \frac{\beta \varepsilon}{2} \\
&\leq \left(2 - \frac{\beta}{\alpha}\right) \left( F(x^*) - F(x_t) \right) + \frac{\beta - \alpha}{\alpha} \gamma \Delta + \frac{\beta \varepsilon}{2} ,
\end{align*}
where the last inequality follows from Cauchy-Schwartz and the assumptions on $\norm{\nabla F(x^*)}$ and the diameter of $\text{Range}(G)$.  Further, by assumption, $\gamma \Delta \leq O(\varepsilon)$. Rearranging terms, we get:
\[
F(x_{t+1}) - F(x^*) \leq \left(\frac{\beta}{\alpha} - 1\right) \left( F(x_t) - F(x^*) \right) + C \varepsilon .
\]
for some constant $C > 0$.
\end{proof}

This theorem asserts that the distance between the objective function at any iteration to the optimum \emph{decreases by a constant factor} in every iteration. (The decay factor is $\frac{\beta}{\alpha} - 1$, which by assumption is a number between 0 and 1). Therefore, we immediately obtain linear convergence of $\varepsilon$-PGD up to a ball of radius $O(\varepsilon)$:

\begin{corollary}
After $T = O(\log \frac{F(x_0) - F(x^*)}{\varepsilon})$ iterations, $F(x_T) \leq F(x^*) + O(\varepsilon$) .
\end{corollary}

Therefore, the overall running time can be bounded as follows:
\[
\text{Runtime} \leq (T_{\varepsilon-\textsc{Proj}} + T_\nabla) \times \log (1 / \varepsilon ) .
\]

It is noticeable that the analysis of PGD algorithm for linear problem is a special case of the generalized analysis given by Theorem~\ref{thm:pgd}. That is because once we set the $F(\cdot)$ as defined in Eq.~\eqref{eq:setup2}, the RSC for $F(\cdot)$ can be obtained through the \SREC~condition from Eq.~\eqref{eq:prf1}. Similarly, we can use the upper bound of the spectral norm for the Gaussian matrix $A$ to obtain RSS for $F(\cdot)$. 

In Sec.~\ref{sec:exp}, we provide empirical results for solving nonlinear inverse problems using Alg.~\ref{alg:PGD}. Specifically, we consider two nonlinear forward models: a sinusoidal model with $\mathcal{A}(x^*) = Ax^* + sin(Ax^*)$; and a sigmoid model with $\mathcal{A}(x^*) = \text{sigmoid}(Ax^*) = \frac{1}{1+\exp(-Ax^*)}$. While we use the $L_2$-loss as a loss function in the case of sinusoidal model, for the sigmoid nonlinearity, we use a loss function specified as:
$$
\ F(x) = \frac{1}{m}\sum _{i=1}^m \left(\Theta (a_i^Tx) - y_i a_i^Tx\right),
$$
where, $\Theta(\cdot)$ is integral of $\mathcal{A}(\cdot)$, and $a_i$ represents the rows of the measurement matrix $A$. The gradient of the loss can be calculated in closed form:
\begin{equation} 
\nabla F(x^*) = \frac{1}{m} A^T (\text{sigmoid}(Ax) - y). 
\label{eq:sig_loss}
\end{equation}
Such choice of the loss function is inspired by the problem of single-index model (SIM) estimation \cite{negahban2012unified}. \cite{soltani2016fast} also advocates the usage of such a loss function.


\subsection{Contribution III: Phase retrieval}

We now propose a method to solve a different category of nonlinear inverse problems using our overall broad approach. Specifically, we will solve inverse problems where the observations are of the form:
\[
y = | A x | +~\text{noise},
\]
i.e., only magnitude information of the phaseless measurements are retained. This is the well-known phase retrieval setting~\cite{candes2015phase,demanet2014stable}; here the added twist is that $x$ is assumed to obey a generative prior.

One approach to solving this problem is to cast the recovery as a nonlinear optimization problem based on a suitable loss function $F(x)$ that measures the (squared) difference between the measurements and the purported phaseless observations. We will slightly depart from this and instead propose a new algorithm (Phase-PGD) for solving the above problem. See Alg.\ \ref{alg:phase_gan}. In each iteration of the Phase-PGD algorithm , three steps are performed: a phase update step, a gradient descent update step, and a projection step. 

The first step is to calculate the phase of $\A\x$. For real $\A$ and $\x$, at the $t^{th}$ iteration, we update the phase estimate:
\[
\p_t = \text{phase}(\A\x_t) \coloneqq \sign{\A\x_t}.
\] 
After calculating the phase vector $\p$, we can use an element-wise product between $\p$ and $\y$ as an estimate of linear measurements and convert the phase retrieval problem into a linear inverse problem.

The second step is simply an application of a gradient descent update rule on the loss function $f(\cdot)$ which is given as:
\[
f(\x) \coloneqq \|\y\odot \p-\A\x\|^2.
\] 
Thus, the gradient descent update at the $t^{th}$ iteration is given by:
\[
\w_t \leftarrow \x_t + \eta \A^T(\y\odot \p_t-\A\x_t),
\] 
where $\eta$ is the learning rate.

The third steps is the projection step, in which we aim to find an image from the span of the generator, $\mathcal{M}$  which is closest to our current estimate $\w_t$. 
We define the projection operator $\mathcal{P}_G$ as follows:
\[
\mathcal{P}_G\left(\w_t\right) \coloneqq G\left(\argmin_{\z}L_{in}(\z)\right),
\]
where $L_{in}$ is the inner loss function defined as,
\[
L_{in}(\z) \coloneqq \|\w_t - G(\z)\|^2.
\]
We solve the inner optimization problem by running gradient descent with $T_{in}$ number of updates on $L_{in}(\z)$. 

\begin{algorithm}[t]
	\caption{\textsc{Phase-PGD}}
	\label{alg:phase_gan}
	\begin{algorithmic}[1]
	\State \textbf{Inputs:} $\y$, $\A$, $G$, $T$, \textbf{Output:}  $\widehat{\x}$
	\State \text{Choose an initial point $\x_0 \in \mathbb{R}^n$} \hspace{17.8em} 
	\For {t = 1,\ldots T }
	\State $\p_{t-1}\leftarrow \sign{\A\x_{{t-1}}}$ 
	\State $\w_{t-1} \leftarrow \x_{t-1} + \eta \A^T(\y\odot \p_{t-1}-\A\x_{t-1})$ \hspace{8.8em} 
	\State $\x_{t} \leftarrow \mathcal{P}_G(\w_{t-1}) = G\left(\argmin_\z \|\w_{t-1} - G(\z)\|\right)$ \hspace{0.6em} 
	\EndFor
	\State $\widehat{\x} \leftarrow \x_{T}$
	\end{algorithmic}
\end{algorithm}

We now analyze this algorithm. 
This part of the algorithm is described in Lines 3-7 of Algorithm~\ref{alg:phase_gan}. We can prove that {provided a good initial estimate ($\x_0$), Phase-PGD provably converges to $\x^*$.} The high level intuition is as follows. Ignoring the noise, the observation model for phase retrieval can be restated as follows:
\begin{align*}
\sign{\iprod{\ai}{\xo}}\odot y_i = \iprod{\ai}{\xo} ,
\end{align*}
for all $i=\{1,2,\ldots, m\}$. To ease notation, denote the \emph{phase vector} ${p} \in \reals^m$ as a vector that contains the unknown signs of the measurements, i.e., ${p}_i = \sign{\iprod{\ai}{\x}}$ for all $i=\{1,2,\ldots,m\}$. Let ${p}^*$ denote the true phase vector and let $\mathcal{P}$ denote the set of all phase vectors, i.e. $\mathcal{P} = \cbrak{\vect{p}:p_i = \pm 1, \forall i}$. Then our measurement model gets modified as:
\begin{align*}
{p}^*\odot\y = \A \xo.
\end{align*}

Therefore, the recovery of $\xo$ can be posed as a (non-convex) optimization problem:
\begin{align} \label{eq:lossfunc}
\min_{\x \in \subspaces,\vect{p} \in \mathcal{P}} \norm{\A\x - {p}\odot\y}^2
\end{align}

To solve this problem, we alternate between estimating $\vect{p}$ and $\x$.
We perform two estimation steps: 

\begin{enumerate}
\item if we fix the signal estimate $\x$, then the minimizer $\vect{p} \in \mathcal{P}$ is given in closed form as:
	\begin{align} \label{eq:phase_est}
	{p}=\sign{\A\x} ,
	\end{align} 
\item and if we fix the phase vector $\vect{p}$, the signal vector $\x \in \subspaces$ can be obtained by solving:
	\begin{align} \label{eq:loss_min}
	\min_{\x \in \subspaces} \|\A \x-{p}\odot \y\|_2 .
	\end{align}
\end{enumerate}

We now analyze our proposed descent scheme. We obtain:

\begin{restatable}{theorem}{convergence}
\label{thm:lin_convergence}
	Suppose we have an initialization $\x_0 \in \subspaces$ satisfying $\distop{\x_0}{\xo} \leq \delta_0 \twonorm{\xo}$, for $0 < \delta_0 < 1$, and suppose the number of (Gaussian) measurements,
	$$ m >  C \rbrak{k d \log n}, $$ 
	for some large enough constant $C$. 
	Then with high probability the iterates $\x_{t+1}$ of Algorithm~\ref{alg:phase_gan}, satisfy:
	\begin{align} \label{eq:mainconvergence}
	\distop{\x_{t+1}}{\xo} \leq {\rho}\, \distop{\x_{t}}{\xo},
	\end{align}
	where $\x_t,\x_{t+1}, \xo \in \mathcal{M}$, and $ 0 < \rho < 1$ is a constant.
\end{restatable}


\noindent{\textbf{Proof sketch:}} The high level idea behind the proof is that with a $\delta$-ball around the true signal $\xo$, the ``phase noise'' can be suitably bounded in terms of a constant times the signal estimation error. To be more precise, suppose that $\z^* = \A \xo = \p^*\odot\y$. Then, at any iteration $t$, we have:
\begin{align*}
\z_t &= \p_t \odot \y \\
&= \p^* \odot \y + (\p_t - \p^*) \odot \y\\
&= \z^* + \e_t,
\end{align*}
where $\e_t$ can be viewed as the ``phase noise''. Now, examining Line 6 of the above algorithm, we have that $\x_t$ is the output of Phase-PGD after $t$ iterations. An ``unpacking'' argument similar to the one in \cite{Jagatap2017,jagatap2019sample} indicates that:
\[
\norm{\x_t - \xo} \leq \alpha \norm{\x_{t-1} - \xo} + \beta \norm{\e_t} ,
\]
where $\alpha$ is a small enough constant. We will show that $\norm{\e_t}$ can be also bounded in terms of $\norm{\x_{t-1} - \xo}$, via Lemma \ref{lem:phase_err_bound} below. Consequently: 
\[
\norm{\x_t - \xo} \leq \rho\norm{\x_{t-1} - \xo},
\]
where $\rho$ is a small enough constant. 

We therefore achieve a per-step error reduction scheme if the initial estimate $x_0$ satisfies $\norm{x_0 - \x^*} \leq \delta_0 \norm{\x^*}$. This result can be trivially extended to the case where the initial estimate $x_0$ satisfies $\norm{x_0 + \x^*} \leq \delta_0 \norm{\x^*}$, hence giving the convergence criterion of the form (for $\rho < 1$): 
\begin{align*}
\distop{\x_{t}}{\xo} \leq \rho\, \distop{\x_{t-1}}{\xo}.
\end{align*}

We now state Lemma \ref{lem:phase_err_bound} without proof; this is a straightfoward  adaptation of the covering number analysis of analysis of \cite{bora2017compressed}.

\begin{restatable}{theorem}{phaseerror} \label{lem:phase_err_bound}
	Suppose that the generator network model $G(\cdot)$ is comprised of $d$ layers of neurons with ReLU activation functions and weight matrices with bounded operator norms. As long as the initial estimate is a small distance away from the true signal $\xo \in \mathcal{M}$ ( i.e. $\distop{\x_0}{\xo} \leq \delta \norm{\xo}$)
	and subsequently,
	$\distop{x_t}{\xo} \leq \delta \norm{\xo}$, where $\x_t$ is the $t^{th}$ update of Algorithm~\ref{alg:phase_gan}, then the following bound holds for any $t \geq 0$:
	\begin{align*}
	\norm{\e_{t+1}} \leq  \rho_1 \norm{\x_t - \xo},
	\end{align*}
	with high probability, as long as $ m > C (kd \log n)$ and $\rho_1 < 1$ is a constant.
\end{restatable}

\subsection{Contribution IV: Addressing signal model mismatch}

A key assumption in all the above discussion has been that the underlying signal is well-modeled by the generator $G$. This, of course, is idealistic; natural signals (and images) exhibit a wide variety of features, some of which may not be modeled by (even) a well-trained generator.

We now generalize the $\varepsilon$-PGD algorithm to handle situations involving signal model mismatch. Assume that the target signal can be decomposed as:
\[
x^* = G(z) + v,
\]
where $\norm{B^T v}_0 \leq l \ll n$ for some ortho-basis $B$. 

For this model, we attempt to solve a (slightly) different optimization problem:
\begin{align}
\widehat{x} &= \argmin~F(x),~~\label{eq:mcop}\\
&\text{s. t.}~~~x = G(z) + v,,\nonumber \\
&~~~~~~~~\norm{B^T v}_0 \leq l.
\end{align}

We propose a new algorithm to solve this problem that we call \emph{Myoptic $\varepsilon$-PGD}. This algorithm is given in Alg.\ \ref{alg:mPGD}\footnote{The algorithm is a variant of block-coordinate descent, except that the block updates share the same gradient term.}.

\begin{algorithm}[t]
	\caption{\textsc{Myopic $\varepsilon$-PGD}}
	\label{alg:mPGD}
	\begin{algorithmic}[1]
	\State \textbf{Inputs:} $y$, $T$, $\nabla$; \textbf{Output:}  $\widehat{x}$
	\State $x_0, u_0, v_0 \leftarrow \textbf{0}$ \hspace{17.8em} 
	\While {$t < T$}
	\State $u_{t+1} = P_G(u_t - \eta \nabla_x F(x_t))$
	\State $v_{t+1} = \text{Thresh}_{B,l}(v_t - \eta \nabla_x F(x_t))$
	\State $x_{t+1} = u_{t+1} + \nu_{t+1}$
	\State $t \leftarrow t+1$
	\EndWhile
	\State $\widehat{x} \leftarrow x_{T}$
	\end{algorithmic}
\end{algorithm}

\begin{theorem}
\label{thm:mpgd}
Let $\oplus$ denote the Minkowski sum. If $F$ satisfies RSC/RSS over $\text{Range}(G) \oplus \text{Span}(B)$ with constants $\alpha$ and $\beta$, and if we assume $\mu$-incoherence between $B$ and $\text{Range}(G)$, we have:
\[
F(x_{t+1}) - F(x^*) \leq \left(\frac{2 - \frac{\beta}{\alpha} \frac{1 - 2.5 \mu}{1 - \mu}}{1 - \frac{\beta}{2\alpha} \frac{\mu}{1 - \mu}}\right) (F(x_t) - F(x^*)) + O(\varepsilon) \, .
\]
\end{theorem}
\bigskip
\begin{proof}
We will generalize the proof technique of~\cite{spinIT}.
We first define some auxiliary variables that help us with the proof. Let:
\begin{align*}
w_t &= x_t - \eta \nabla F(x_t), \\
w_t^u &= u_t - \eta \nabla F(x_t), \\
w_t^v &= v_t - \eta \nabla F(x_t) .
\end{align*}
and let $x^* = u^* + v^*$ be the minimizer that we seek. 
As above, by invoking RSS and with some algebra, we obtain:
\begin{align}
F(x_{t+1}) - F(x^*) \leq \frac{\beta}{2} \left(\norm{x_{t+1} - w_t}^2 - \norm{x_t - w_t}^2 \right),
\label{eq:upperb_rsc}
\end{align}
However, by definition, 
\begin{align*}
x_{t+1} &= u_{t+1} + v_{t+1}, \\
x_{t} &= u_{t} + v_{t} .
\end{align*}
Therefore, 
\begin{align*}
&\norm{x_{t+1} - w_t}^2  \\
& = \| u_{t+1} - (u_t - \eta \nabla F(x_t)) + \\
&~~~~v_{t+1} - (v_t - \eta \nabla F(x_t)) + \eta \nabla F(x_t) \|^2 \\
& = \| u_{t+1} - (u_t - \eta \nabla F(x_t)) \|^2 + \| \eta \nabla F(x_t) \|^2 +\\ 
&~~~~ \| v_{t+1} - (v_t - \eta \nabla F(x_t)) \|^2 \\
&~~+ 2 \langle u_{t+1} - (u_t - \eta \nabla F(x_t)), \eta \nabla F(x_t) \rangle \\
&~~+ 2 \langle v_{t+1} - (v_t - \eta \nabla F(x_t)), \eta \nabla F(x_t) \rangle \\
&~~+ 2 \langle u_{t+1} - (u_t - \eta \nabla F(x_t)), v_{t+1} - (v_t - \eta \nabla F(x_t)) \rangle .
\end{align*}
But $u_{t+1}$ is an $\varepsilon$-projection of $w_t^u$ and $u^*$ is in the range of $G$, we have: 
\[
\norm{u_{t+1} - w_t^u}^2 \leq \norm{u^* - w^u_t}^2 + \varepsilon.
\]
Similarly, since $v_{t+1}$ is an $l$-sparse thresholded version of $w_t^v$, we have:
\[
\norm{v_{t+1} - w_t^v}^2 \leq \norm{v^* - w^v_t}^2 .
\]
Plugging in these two upper bounds, we get:
\begin{align*}
&\norm{x_{t+1} - w_t}^2  \\
& \leq \| u^* - (u_t - \eta \nabla F(x_t)) \|^2 + \varepsilon \\ 
&~~+ \| \eta \nabla F(x_t) \|^2 + \| v^* - (v_t - \eta \nabla F(x_t)) \|^2 \\
&~~+ 2 \langle u_{t+1} - (u_t - \eta \nabla F(x_t)), \eta \nabla F(x_t) \rangle \\
&~~+ 2 \langle v_{t+1} - (v_t - \eta \nabla F(x_t)), \eta \nabla F(x_t) \rangle \\
&~~+ 2 \langle u_{t+1} - (u_t - \eta \nabla F(x_t)), v_{t+1} - (v_t - \eta \nabla F(x_t)) \rangle .
\end{align*}
Expanding squares and cancelling (several) terms, the right hand side of the above inequality can be simplified to obtain:
\begin{align*}
&\norm{x_{t+1} - w_t}^2 \\ 
&\leq \norm{u^* + v^* - w_t}^2 +  \varepsilon  \\
&~~~~ + 2 \langle u_{t+1} - u_t, v_{t+1} - v_t \rangle - 2 \langle u^* - u_t, v^* - v_t \rangle \\
&= \norm{x^* - w_t}^2  + \varepsilon + 2 \langle u_{t+1} - u_t, v_{t+1} - v_t \rangle \\
&~~~~- 2 \langle u^* - u_t, v^* - v_t \rangle .
\end{align*}
Plugging this into \eqref{eq:upperb_rsc}, we get:
\begin{align*}
& F(x_{t+1}) - F(x^*) \\
& \leq \underbrace{\frac{\beta}{2} \left( \norm{x^* - w_t}^2 - \norm{x_t - w_t}^2 \right)}_{\mathbb{T}_1}  \\
&~~+ \underbrace{\beta \left( \langle u_{t+1} - u_t, v_{t+1} - v_t \rangle - 2 \langle u^* - u_t, v^* - v_t \rangle  \right)}_{\mathbb{T}_2} \\
&~~+ \frac{\beta \varepsilon}{2}.
 \end{align*}
We already know how to bound the first term $\mathbb{T}_1$, using an identical argument as in the proof of Theorem~\ref{thm:pgd}. We get:
\[
\mathbb{T}_1 \leq \left(2 - \frac{\beta}{\alpha}\right) \left( F(x^*) - F(x_t) \right) + \frac{\beta - \alpha}{\alpha} \gamma \Delta .
\]
The second term $\mathbb{T}_2$ can be bounded as follows. First, observe that
\begin{align*}
& | \langle u_{t+1} - u_t, v_{t+1} \rangle | \\
&\leq \mu \norm{u_{t+1} - u_t} \norm{v_{t+1} - v_t} \\
&\leq \frac{\mu}{2} \left( \norm{u_{t+1} - u_t}^2 + \norm{v_{t+1} - v_t}^2 \right) \\
&\leq \frac{\mu}{2} \left( \norm{u_{t+1} + v_{t+1} - u_t - v_t}^2 \right) \\
&~~~~+ \mu | \langle u_{t+1} - u_t, v_{t+1} - v_t \rangle | .
\end{align*}
This gives us the following inequalities:
\begin{align*}
& | \langle u_{t+1} - u_t, v_{t+1} \rangle | \\
&\leq \frac{\mu}{2(1 - \mu)} \norm{x_{t+1} - x_t}^2 \\
&= \frac{\mu}{2(1 - \mu)} \Big(  \norm{x_{t+1} - x^*}^2 +  \norm{x_{t} - x^*}^2 + \Big. \\
& \Big. ~~~~~~~~~~~~~~~~~~~~2 | \langle x_{t+1} - x^*, x_t - x^* \rangle  |  \Big) \\
&\leq \frac{\mu}{1 - \mu} \left( \norm{x_{t+1} - x^*}^2 +  \norm{x_{t} - x^*}^2  \right) .
\end{align*}
Similarly,
\begin{align*}
| \langle u^* - u_t, v^* - v_t \rangle | &\leq \mu \norm{u^* - u_t} \norm{v^* - v_t} \\
&\leq \frac{\mu}{2} \left( \norm{u^* - u_t}^2 + \norm{v^* - v_t} \right) \\
&= \frac{\mu}{2} \left( \norm{u^* + v^* - u_t - v_t}^2 \right) \\
&~~~~+ \mu | \langle u^* - u_t, v^* - v_t \rangle | ,
\end{align*}
which gives:
\[
| \langle u^* - u_t, v^* - v_t \rangle | \leq \frac{\mu}{2(1- \mu)} \norm{x^* - x_t}^2 .
\]
Combining, we get:
\begin{align*}
\mathbb{T}_2 &\leq \frac{\beta \mu}{2(1- \mu)} \left( 3\norm{x^* - x_t}^2 + \norm{x^* - x_{t+1}}^2 \right) .
\end{align*}
Moreover, by invoking RSC and Cauchy-Schwartz (similar to the proof of Theorem~\ref{thm:pgd}), we have:
\begin{align*}
\norm{x^* - x_t}^2 &\leq \frac{1}{\alpha} \left( F(x_t) - F(x^*) \right) + O(\varepsilon) ,\\
\norm{x^* - x_{t+1}}^2 &\leq \frac{1}{\alpha} \left( F(x_{t+1}) - F(x^*) \right) + O(\varepsilon) .
\end{align*}
Therefore we obtain the upper bound on $\mathbb{T}_2$:
\begin{align*}
\mathbb{T}_2 &\leq \frac{3 \beta \mu}{2 \alpha (1 - \mu)} \left( F(x_t) - F(x^*) \right) \\
&~~~ + \frac{\beta \mu}{2 \alpha (1 - \mu)} \left( F(x_{t+1}) - F(x^*) \right) + C' \varepsilon .
\end{align*}
Plugging in the upper bounds on $\mathbb{T}_1$ and $\mathbb{T}_2$ and re-arranging terms, we get:
\begin{align*}
&\left(1 - \frac{\beta \mu}{2\alpha (1 - \mu)} \right) (F(x_{t+1}) - F(x^*)) \\
&\leq \left(2 - \frac{\beta}{\alpha} + \frac{3 \beta \mu}{2 \alpha (1 - \mu)}  \right) (F(x_{t}) - F(x^*)) + C' \varepsilon,
\end{align*}
which leads to the desired result.
\end{proof}

%% file: common/math_model.tex
\subsection{Contribution I: Solving linear inverse problems}
\label{sec:setup}
Let $\S \subseteq \R^n$ be the set of `natural' images in data space with a vector $x^* \in \S$. We consider an ill-posed linear inverse problem \eqref{eq:lip} with the linear operator $\mathcal{A}(x) = Ax$, where $A$ is a Gaussian random matrix. For simplicity, we do not consider the additive noise term. 
\begin{align}
y = Ax^*,~~\label{eq:lip}
\end{align}
To solve for $\widehat{x}$ (estimate of $x^*$), we choose Euclidean measurement error as the loss function $F(\cdot)$ in Eqn. \eqref{eq:cop}. Therefore, given $y$ and $A$, we seek
\begin{align}
\widehat{x} = \argmin_{x \in \S}\|y-Ax\|^2.
\label{eq:setup2}
\end{align} 
\subsubsection{Algorithm}
Our algorithm is described in Alg.~\ref{alg:linear-PGD}. 
We assume that our trained generator network (G) well approximates the high-dimensional probability distribution of the set $\S$. With this assumption, we limit our search for $\widehat{x}$ only to the range of the generator function ($G(z)$). The function $G$ is assumed to be differentiable, and hence we use back-propagation for calculating the gradients of the loss functions involving $G$ for gradient descent updates.

The optimization problem in Eqn. \ref{eq:setup2} is similar to a least squares estimation problem, and a typical approach to solve such problems is to use gradient descent. However, the candidate solutions obtained after each gradient descent update need not represent a `natural' image and may not belong to set $\S$. We solve this limitation by projecting the candidate solution on the range of the generator function after each gradient descent update. 

Thus, in each iteration of our proposed algorithm \ref{alg:linear-PGD}, two steps are performed in alternation: a gradient descent update step and a projection step. 
The first step is simply an application of a gradient descent update rule on the loss function $F(\cdot)$ with the learning rate $\eta$.
In projection step, we minimize the projection loss by gradient descent updates with learning rate $\eta_{in}$: 
\[
P_G\left(w_t\right) \coloneqq G\left(\argmin_{z}\|w_t - G(z)\|\right),
\]
Though the projection loss function is highly non-convex due to the presence of $G$, we find empirically that the gradient descent (implemented via back-propagation) works very well. Thus, the gradient descent based minimization serves as a projection oracle in this case. 
In each of the $T$ iterations, we run $T_{in}$ gradient descent updates for calculating the projection. Therefore, $T \times T_{in}$ is the total number of gradient descent updates required in our approach.

\begin{algorithm}[t]
	\caption{\textsc{PGD}}
	\label{alg:linear-PGD}
	\begin{algorithmic}[1]
	\State \textbf{Inputs:} $y$, $A$, $G$, $T$, \textbf{Output:}  $\widehat{x}$
	\State $x_0 \leftarrow \textbf{0}$ \hspace{17.8em} 
	\While {$t < T$}
	\State $w_t \leftarrow x_t + \eta A^T(y-Ax_t)$ \hspace{8.8em} 
	\State $x_{t+1} \leftarrow G\left(\argmin_{z}\|w_t - G(z)\|\right)$ \hspace{0.6em} 
	\State $t \leftarrow t+1$
	\EndWhile
	\State $\widehat{x} \leftarrow x_{T}$
	\end{algorithmic}
\end{algorithm}

%% file: common/theory_results.tex
\subsubsection{Analysis}
Drawing parallels with standard compressive sensing theory, in our case, we need to ensure that the difference vector of any two signals in the set $\S$ lies away from the nullspace of the matrix $A$. This condition is encoded via the \SREC~(Set Restricted Eigenvalue Condition) as defined and established in \cite{bora2017compressed}. We slightly modify this condition and present it in the form of squared $\ell_2$-norm : 
\begin{definition}
	Let $\S \in \R^n$. $A$ is $m \times n$ matrix. For parameters $\gamma > 0,~\delta \geq 0$, matrix $A$ is said to satisfy the \SREC$(\S, \gamma, \delta)$ if,
	\[
	\|A(x_1-x_2)\|^2 \geq \gamma \|x_1-x_2\|^2 - \delta,
	\]
	for $\forall x_1,x_2 \in \S$.
\end{definition}
Further, based on \cite{shah2011iterative,foucart2013}, we propose the following theorem about the convergence of our algorithm:
\begin{theorem}
Let $G: \R^k \rightarrow \R^n$ be a differentiable generator with $d$ layers and range $\S$. Let $A$ be a random Gaussian matrix with $A_{i,j}\sim N(0,1/m)$ with $m \geq C (k d \log n)$ for some positive constant $C$.
Then, for every vector $x^* \in \S$, the sequence $\left(x_t\right)$ defined by the algorithm {PGD} [\ref{alg:linear-PGD}] exhibits linear convergence for a carefully chosen range of stepsizes $\eta$.
\end{theorem}

\begin{proof}
	Suppose $F(\cdot)$ is the squared error loss function as defined above. Then, we have:
	\begin{align*}
	& F(x_{t+1}) - F(x_t) \\
	& = \|Ax_{t+1}\|^2 - 2\langle y, Ax_{t+1} \rangle + 2\langle y, Ax_t \rangle - \|Ax_t\|^2, \\
	&= \|Ax_{t+1} - Ax_t\|^2 + 2\langle x_t-x_{t+1}, A^{T}A(x^*-x_t) \rangle.
	\end{align*}
	Substituting $y = Ax^*$ and rearranging yields,
	\begin{align}
	2\langle x_t - x_{t+1}, A^T(y-Ax_t) \rangle & = F(x_{t+1}) - F(x_t) \nonumber \\
	& - \|Ax_{t+1} - Ax_t\|^2.
	\label{eq:prf1}
	\end{align}
	Define:
	\begin{align*}
		w_t \coloneqq x_t + \eta A^T(y-Ax_t) = x_t + \eta A^TA(x^*-x_t)
	\end{align*}
	Then, by definition of the projection operator $P_G$, the vector $x_{t+1}$ is a better (or equally good) approximation to $w$
	as the true image $x^*$. Therefore, we have:
	\begin{align*}
	\|x_{t+1} - w_t\|^2 \leq \|x^* - w_t\|^2.
	\end{align*}
	Substituting for $w_t$ and expanding both sides, we get:
	\begin{align*}
	& \|x_{t+1} - x_t\|^2 - 2\eta \langle x_{t+1}-x_t, A^T(y-Ax_t) \rangle \\
	& \leq \|x^* - x_t\|^2 - 2\eta \langle x^*-x_t, A^T(y-Ax_t) \rangle.
	\end{align*}
	Substituting $y = Ax^*$ and rearranging yields,
	\begin{align}
	& 2\langle x_t - x_{t+1}, A^T(y-Ax_t) \rangle \nonumber \\
	& \leq \frac{1}{\eta}\|x^* - x_t\|^2 - \frac{1}{\eta}\|x_{t+1} - x_t\|^2 - 2F(x_t).
	\label{eq:prf2}
	\end{align}
	We now use \ref{eq:prf1} and \ref{eq:prf2} to obtain,
	\begin{align}
	& F(x_{t+1}) \leq \frac{1}{\eta}\|x^* - x_t\|^2 - F(x_t) \nonumber \\
	& - \left( \frac{1}{\eta}\|x_{t+1} - x_t\|^2 - \|Ax_{t+1} - Ax_t\|^2 \right).
	\label{eq:prf3}
	\end{align}
	Now, since $A$ is a random Gaussian with sufficiently many rows, it satisfies the \SREC~\cite{bora2017compressed}:  
	$$\|A (x_1 - x_2)\|^2 \geq \gamma \|x_1 - x_2 \|^2 - \delta.$$ 
	As $x^*, x_t$ and $x_{t+1}$ are `natural' vectors,
	\begin{align}
	\frac{1}{\eta} \|x^* -x_t\|^2 \leq \frac{1}{\eta \gamma}\|y-Ax_t\|^2 + \frac{\delta}{\eta \gamma}.
	\label{eq:prf4}
	\end{align}
	Substituting \ref{eq:prf4} in \ref{eq:prf3},
	\begin{align*}
	& F(x_{t+1}) \leq \left(\frac{1}{\eta \gamma} -1\right) F(x_t)\\
	&~~~~~~~-\left( \frac{1}{\eta}\|x_{t+1} - x_t\|^2 - \|Ax_{t+1} - Ax_t\|^2 \right) + \frac{\delta}{\eta \gamma}.
	\end{align*}
	
	Morever, as shown in~\cite{bora2017compressed}, we have $\|Av\| \leq \rho \|v\|$ with high probability for any $v$. Therefore, we write:
	$$\|Ax_{t+1} - Ax_t\|^2 \leq \rho^2\|x_{t+1} - x_t\|^2,$$
	$$ \|Ax_{t+1} - Ax_t\|^2 - \frac{1}{\eta}\|x_{t+1} - x_t\|^2 \leq \left(\rho^2 - \frac{1}{\eta}\right)\|x_{t+1} - x_t\|^2.$$
	Let us choose learning rate$(\eta)$ such that $\frac{1}{2\gamma}<\eta < \frac{1}{\gamma}$. We also have $\rho^2 \leq \gamma$. Combining both, we get $\rho^2 < \frac{1}{\eta}$, which makes the L.H.S. in the above equation negative. Therefore,
	$$F(x_{t+1}) \leq \left(\frac{1}{\eta \gamma} -1\right) F(x_t) + \frac{\delta}{\eta \gamma},$$
	where $\delta$ is inversely proportional to the number of measurements $m$ \cite{bora2017compressed}. If the slack parameter $\delta$ is small enough then it can be ignored. Also, $\frac{1}{2\gamma}<\eta < \frac{1}{\gamma}$ yields,
	$$0<\left(\frac{1}{\eta \gamma} -1\right) < 1.$$
	Hence,
	\begin{align}
	F(x_{t+1}) \leq \alpha F(x_t) + \delta; ~0< \alpha < 1,
	\label{eq:prf5}
	\end{align}
    i.e., the sequence $F(x_t)$ converges linearly upto a neighborhood of radius $O(\delta)$. In the noiseless case, this also implies that $x_{t+1}$ converges to a small neighborhood around $x^*$.
\end{proof}

%% file: common/figure2.tex
\begin{figure*}[!t]
	\begin{center}
		\begingroup
		\setlength{\tabcolsep}{4pt} 
		\renewcommand{\arraystretch}{1} 
		\begin{tabular}{ccc}      
			\raisebox{-0.6\height}{
			\begin{tikzpicture}[scale=0.65]
			\pgfplotsset{scaled y ticks=false}
			
			\begin{axis}[
			xlabel=Number of measurements $(m)$,
			ylabel=Reconstruction error (per pixel), 
			grid=both,
			minor y tick num=1,
			major grid style={dashed},
			minor grid style={dotted},
			ymin=0,
			xtick distance=40,
			yticklabel style={
				/pgf/number format/fixed,
				/pgf/number format/precision=5
			}] 
			\addplot[black!60!green, mark=diamond*, mark size=2.5] 
			coordinates {
				(20,0.1116925674)
				(40,0.1104524388)
				(60,0.1077322946)
				(80,0.1043693087)
				(100,0.1009670696)
				(120,0.09845475989)
				(140,0.09478937129)
				(160,0.08939492804)
				(180,0.08597683983)
				(200,0.07938231549)}; 
			\addplot[blue, mark=*, mark size=2] 
			coordinates{
				
				(20,0.0690527727406)
				(40,  0.0392738807791)
				(60, 0.0335496354056)
				(80, 0.0301922520026)
				(100,  0.022366575626)
				(120, 0.0189172737835)
				(140, 0.0185702356315)
				(160, 0.0134109912651)
				(180, 0.0109117549589)
				(200,0.0108717565358)}; 
			\addplot[red, mark=square*, mark size=2] 
			coordinates{
				(20,0.03928229894)
				(40,  0.01578411299)
				(60, 0.01155714582)
				(80, 0.002478766476)
				(100, 0.001179446696)
				(120, 0.0003373235479)
				(140, 0.0001405245319)
				(160,  1.25E-04)
				(180, 6.86E-05)
				(200,5.84E-05) };
			\legend{LASSO,CSGM,PGDGAN} 
			\end{axis} 		
			\end{tikzpicture}}& 
			\raisebox{-0.5\height}{
				\includegraphics[width=0.20\linewidth]{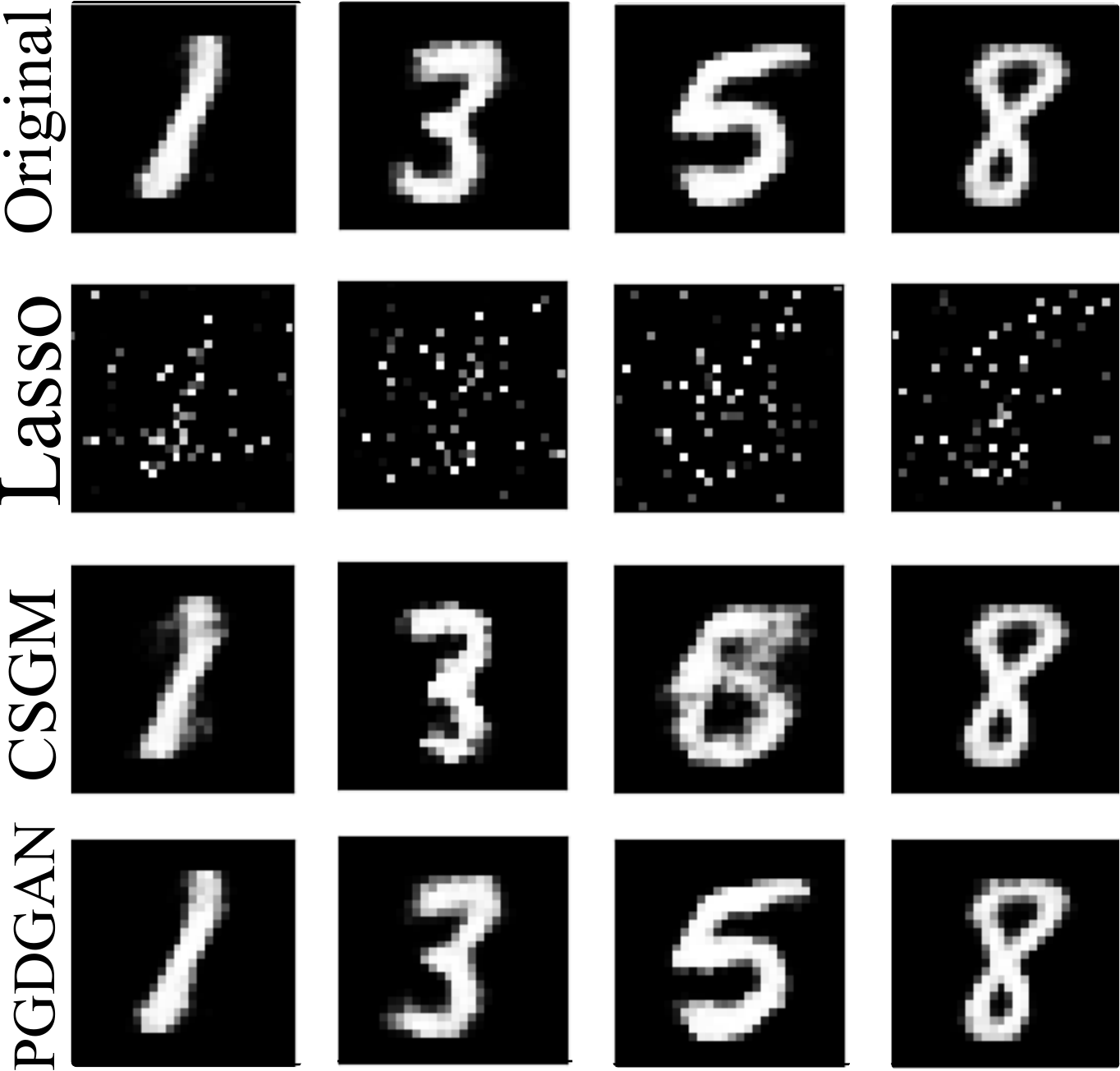}}&
					\begingroup
					\setlength{\tabcolsep}{1pt} 
					\renewcommand{\arraystretch}{1.2} 
					\begin{tabular}{ccccccc}      
						\begin{sideways}{\scriptsize ~~Original}\end{sideways}&
						\includegraphics[width=\s\linewidth]{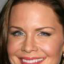}&
						\includegraphics[width=\s\linewidth]{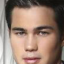}&
						\includegraphics[width=\s\linewidth]{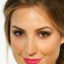}&
						\includegraphics[width=\s\linewidth]{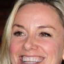}&
						\includegraphics[width=\s\linewidth]{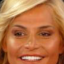}&
						\includegraphics[width=\s\linewidth]{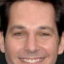}
						\\
						\begin{sideways}{\scriptsize ~~Lasso}\end{sideways}&
						\includegraphics[width=\s\linewidth]{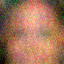}&
						\includegraphics[width=\s\linewidth]{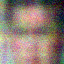}&
						\includegraphics[width=\s\linewidth]{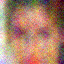}&
						\includegraphics[width=\s\linewidth]{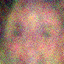}&
						\includegraphics[width=\s\linewidth]{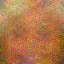}&
						\includegraphics[width=\s\linewidth]{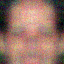}
						\\
						\begin{sideways}{\scriptsize ~~~CSGM}\end{sideways}&		
						\includegraphics[width=\s\linewidth]{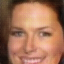}&
						\includegraphics[width=\s\linewidth]{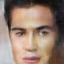}&
						\includegraphics[width=\s\linewidth]{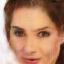}&
						\includegraphics[width=\s\linewidth]{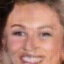}&
						\includegraphics[width=\s\linewidth]{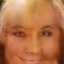}&
						\includegraphics[width=\s\linewidth]{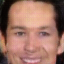}
						\\
						\begin{sideways}{\scriptsize PGDGAN}\end{sideways}&
						\includegraphics[width=\s\linewidth]{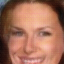}&
						\includegraphics[width=\s\linewidth]{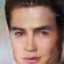}&
						\includegraphics[width=\s\linewidth]{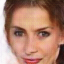}&
						\includegraphics[width=\s\linewidth]{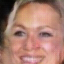}&
						\includegraphics[width=\s\linewidth]{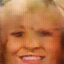}&
						\includegraphics[width=\s\linewidth]{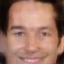}
					\end{tabular}
					\endgroup
					\\
				(a) & (b) & (c)	
		\end{tabular}
		\endgroup
	\end{center}
	\caption{\emph{(a) Comparison of our algorithm (Alg.~\ref{alg:linear-PGD}) with CSGM \cite{bora2017compressed} and Lasso on MNIST; (b) Reconstruction results with $m=100$ measurements; (c) Reconstruction results on celebA dataset with $m=1000$ measurements.}}
	\label{fig:mnist}
\end{figure*}
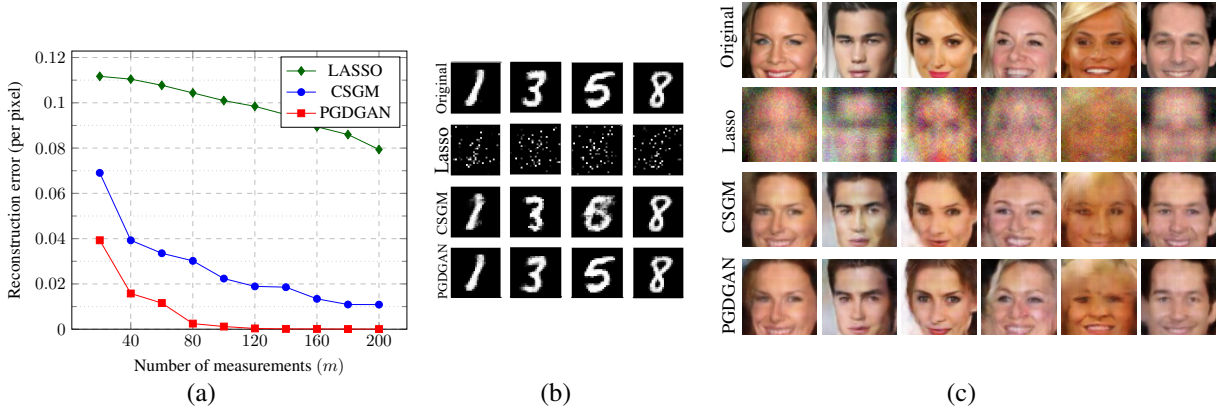

\begin{figure*}[!t]
	\begin{center}
		\begingroup
		\setlength{\tabcolsep}{4pt} 
		\renewcommand{\arraystretch}{1} 
		\begin{tabular}{cc}      
					\begingroup
					\setlength{\tabcolsep}{1pt} 
					\renewcommand{\arraystretch}{1.2} 
					\begin{tabular}{cccccccc}      
						\begin{sideways}{\scriptsize ~~Original}\end{sideways}&
						\includegraphics[width=\s\linewidth]{./fig/ground_truth/3.png}&
						\includegraphics[width=\s\linewidth]{./fig/ground_truth/36.png}&
						\includegraphics[width=\s\linewidth]{./fig/ground_truth/63.png}&
						\includegraphics[width=\s\linewidth]{./fig/ground_truth/46.png}&
						\includegraphics[width=\s\linewidth]{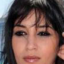}&
						\includegraphics[width=\s\linewidth]{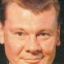}&
						\includegraphics[width=\s\linewidth]{./fig/ground_truth/59.png}
						\\
						\begin{sideways}{\scriptsize ~~Lasso}\end{sideways}&
						\includegraphics[width=\s\linewidth]{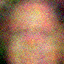}&
						\includegraphics[width=\s\linewidth]{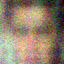}&
						\includegraphics[width=\s\linewidth]{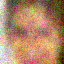}&
						\includegraphics[width=\s\linewidth]{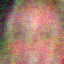}&
						\includegraphics[width=\s\linewidth]{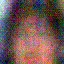}&
						\includegraphics[width=\s\linewidth]{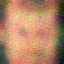}&
						\includegraphics[width=\s\linewidth]{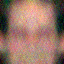}
						\\
						\begin{sideways}{\scriptsize ~~$\eps$-PGD}\end{sideways}&
						\includegraphics[width=\s\linewidth]{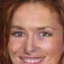}&
						\includegraphics[width=\s\linewidth]{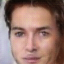}&
						\includegraphics[width=\s\linewidth]{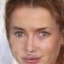}&
						\includegraphics[width=\s\linewidth]{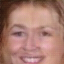}&
						\includegraphics[width=\s\linewidth]{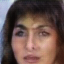}&
						\includegraphics[width=\s\linewidth]{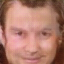}&
						\includegraphics[width=\s\linewidth]{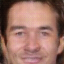}
					\end{tabular}
					\endgroup
		&
					\begingroup
					\setlength{\tabcolsep}{1pt} 
					\renewcommand{\arraystretch}{1.2} 
					\begin{tabular}{cccccccc}      
						\begin{sideways}{\scriptsize ~~Original}\end{sideways}&
						\includegraphics[width=\s\linewidth]{./fig/ground_truth/3.png}&
						\includegraphics[width=\s\linewidth]{./fig/ground_truth/36.png}&
						\includegraphics[width=\s\linewidth]{./fig/ground_truth/63.png}&
						\includegraphics[width=\s\linewidth]{./fig/ground_truth/46.png}&
						\includegraphics[width=\s\linewidth]{./fig/ground_truth/18.png}&
						\includegraphics[width=\s\linewidth]{./fig/ground_truth/27.png}&
						\includegraphics[width=\s\linewidth]{./fig/ground_truth/59.png}
						\\
						\begin{sideways}{\scriptsize ~~Lasso}\end{sideways}&
						\includegraphics[width=\s\linewidth]{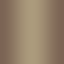}&
						\includegraphics[width=\s\linewidth]{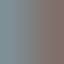}&
						\includegraphics[width=\s\linewidth]{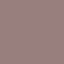}&
						\includegraphics[width=\s\linewidth]{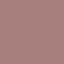}&
						\includegraphics[width=\s\linewidth]{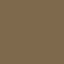}&
						\includegraphics[width=\s\linewidth]{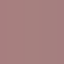}&
						\includegraphics[width=\s\linewidth]{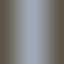}
						\\
						\begin{sideways}{\scriptsize ~~$\eps$-PGD}\end{sideways}&
						\includegraphics[width=\s\linewidth]{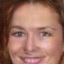}&
						\includegraphics[width=\s\linewidth]{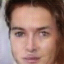}&
						\includegraphics[width=\s\linewidth]{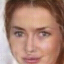}&
						\includegraphics[width=\s\linewidth]{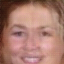}&
						\includegraphics[width=\s\linewidth]{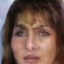}&
						\includegraphics[width=\s\linewidth]{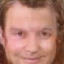}&
						\includegraphics[width=\s\linewidth]{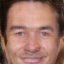}
					\end{tabular}
					\endgroup
					\\
				(a) & (b)
		\end{tabular}
		\endgroup
	\end{center}
	\caption{\emph{Comparison of our algorithm (Alg.~\ref{alg:PGD}) with Lasso for non-linear forward models (a) with $\mathcal{A}(x^*)=Ax^*+ sin(Ax^*)$; (b) with $\mathcal{A}(x^*) = sigmoid(Ax^*)$. Reconstruction results are on celebA dataset with $m=1000$ measurements.}}
	\label{fig:nonlinear}
\end{figure*}

%% file: common/experiment_results.tex
\section{Experimental Results}
\label{sec:exp}
Our primary focus in this paper is theoretical. However, we supplement our theory with representative numerical experiments that show the promise of our proposed algorithms. in this section, we describe our experimental setup and report algorithm performance.

\subsection{Experiments with compressed sensing}

We use two different GAN architectures and two different datasets in our experiments to show that our approach [Alg.~\ref{alg:linear-PGD}] can work with variety of GAN architectures and datasets. We provide comparisons with the CSGM algorithm proposed in ~\cite{bora2017compressed} as well as the Lasso.

In our experiments, we choose the entries of the matrix $A$ independently from a Gaussian distribution with zero mean and variance $1/m$. In these experiments, we ignore the presence of noise; however, our experiments can be replicated in the presence of additive Gaussian noise. For comparison with CSGM~\cite{bora2017compressed}, we use a gradient descent optimizer keeping the total number of update steps $(T \times T_{in})$ fixed for both algorithms to produce fair comparisons.  

In the first experiment, we use a very simple GAN model trained on the MNIST dataset, which is collection of $60,000$ handwritten digit images, each of size $28 \times 28$ \cite{lecun1998gradient}. In our GAN, both the generator and the discriminator are fully-connected neural networks with only one hidden layer. The generator consists of $20$ input neurons, $200$ hidden-layer neurons and $784$ output neurons, while the discriminator consists of $784$ input neurons, $128$ hidden layer neurons and $1$ output neuron. The size of the latent space is set to $k = 20$, i.e., the input to our generator is a standard normal vector $z \in R^{20}$.  We train the GAN using the method described in \cite{goodfellow2014generative}. We use the Adam optimizer \cite{kingma2014adam} with learning rate $0.001$ and mini-batch size $128$ for the training. 
 
We test the MNIST GAN with $10$ test images taken from the span of the generator to avoid model mismatch issues, and provide both quantitative and qualitative results. For PGD-GAN, because of the zero initialization, a high learning rate is required to get a meaningful output before passing it to the projection step. Therefore, we choose $\eta = 0.5$. The parameter $\eta_{in}$ is set to $0.01$ with $T=15$ and $T_{in}=200$. Thus, the total number of update steps is fixed to $3000$. Similarly, the algorithm of \cite{bora2017compressed} is tested with $3000$ updates and $\eta = 0.01$. For reporting purposes, we use the reconstruction error $= \|\widehat{x}-x^*\|^2$. In Fig. \ref{fig:mnist}(a), we show the reconstruction error comparisons for increasing values of number of measurements. It can be seen that our algorithm performs better than the other two methods. Also, as the input images are chosen from the span of the generator itself, it is possible to get close to zero error with only $100$ measurements. Fig. \ref{fig:mnist}(b) depicts reconstruction results for selected MNIST images.
 
The second set of our experiments are performed on a Deep Convolutional GAN (DCGAN) trained on the celebA dataset, which contains more than $200,000$ face images of celebrities \cite{liu2015deep}. We use a pre-trained DCGAN model, which was made available by \cite{bora2017compressed}. The dimension of the latent space for the DCGAN model is $k=100$. We report the results on a held out test dataset, unseen by the GAN at the time of training. The total number of updates is set to $1000$, with $T = 10$ and $T_{in}=100$. Learning rates for PGD-GAN are set as $\eta = 0.5$ and $\eta_{in}=0.1$. The algorithm of \cite{bora2017compressed} is run with $\eta = 0.1$ and $1000$ update steps. Image reconstruction results from $m=1000$ measurements with our algorithm are displayed in Fig.~\ref{fig:mnist}(c). We observe that our algorithm produces better reconstructions compared to the other baselines.

\subsection{Experiments with nonlinear inverse problems}
We extend our experiments to nonlinear models to depict the performance of our algorithm as described in Sec.~\ref{sec:nonlinear}. We present image reconstructions from the measurements obtained using two non-linear forward models: a sinusoidal model with $\mathcal{A}(x^*)=Ax^*+ \sin(Ax^*)$, and a sigmoidal model with $\mathcal{A}(x^*) = \text{sigmoid}(Ax^*)$. Similar to the linear case, these experiments are performed using a DCGAN trained on celebA. It is evident that our algorithm produces superior reconstructions as depicted in Fig.~\ref{fig:nonlinear}.

%% file: common/rakib-icassp.tex
\begin{figure}
\centering
   	\begin{subfigure}[t]{0.40\textwidth}
		\includegraphics[width=\textwidth]{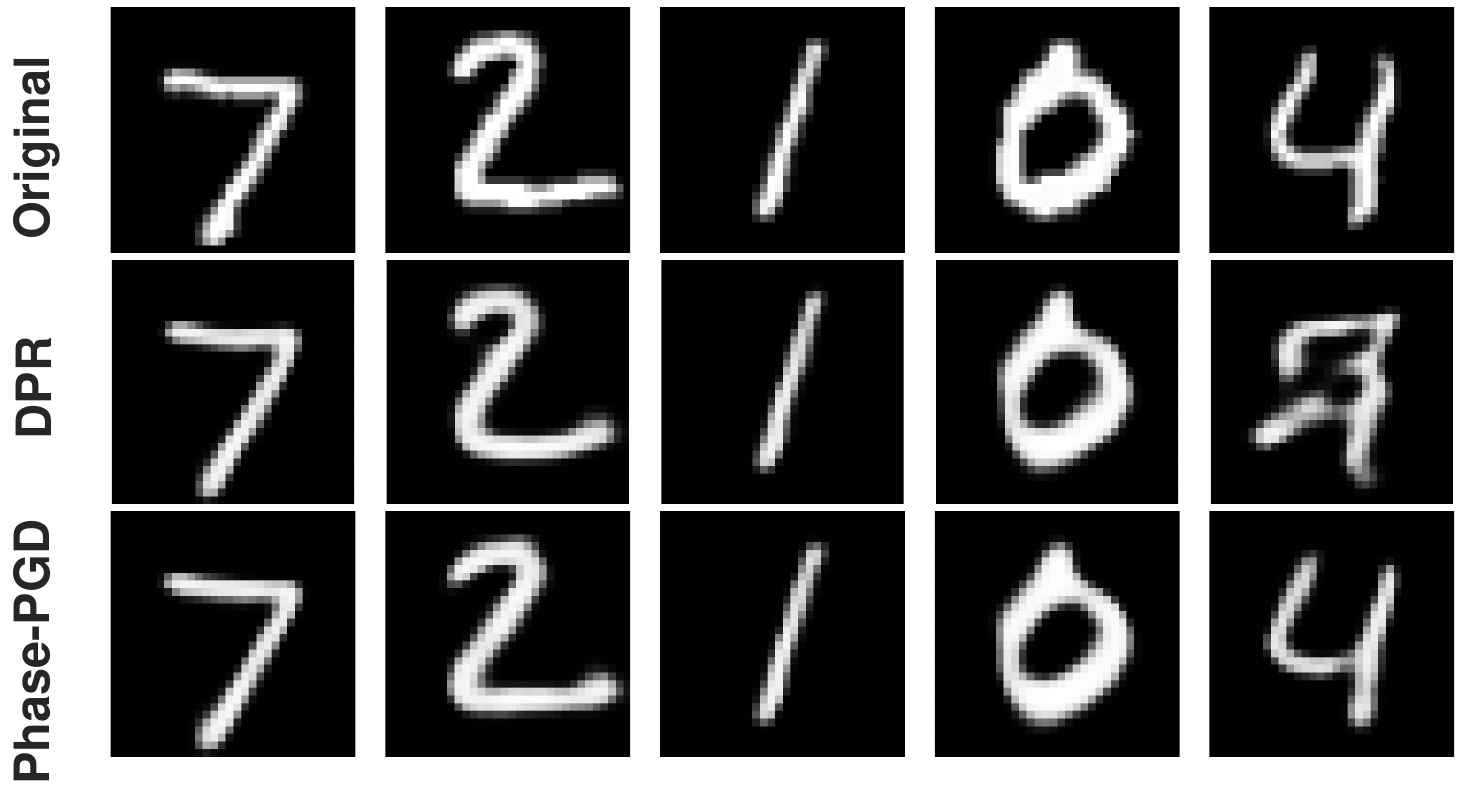}
		\caption{Reconstruction results on MNIST with $m = 60$ measurements.}
		\label{fig:mnist-rec}
		\end{subfigure}\hfill%
	\begin{subfigure}[t]{0.24\textwidth}
		\includegraphics[width=\textwidth]{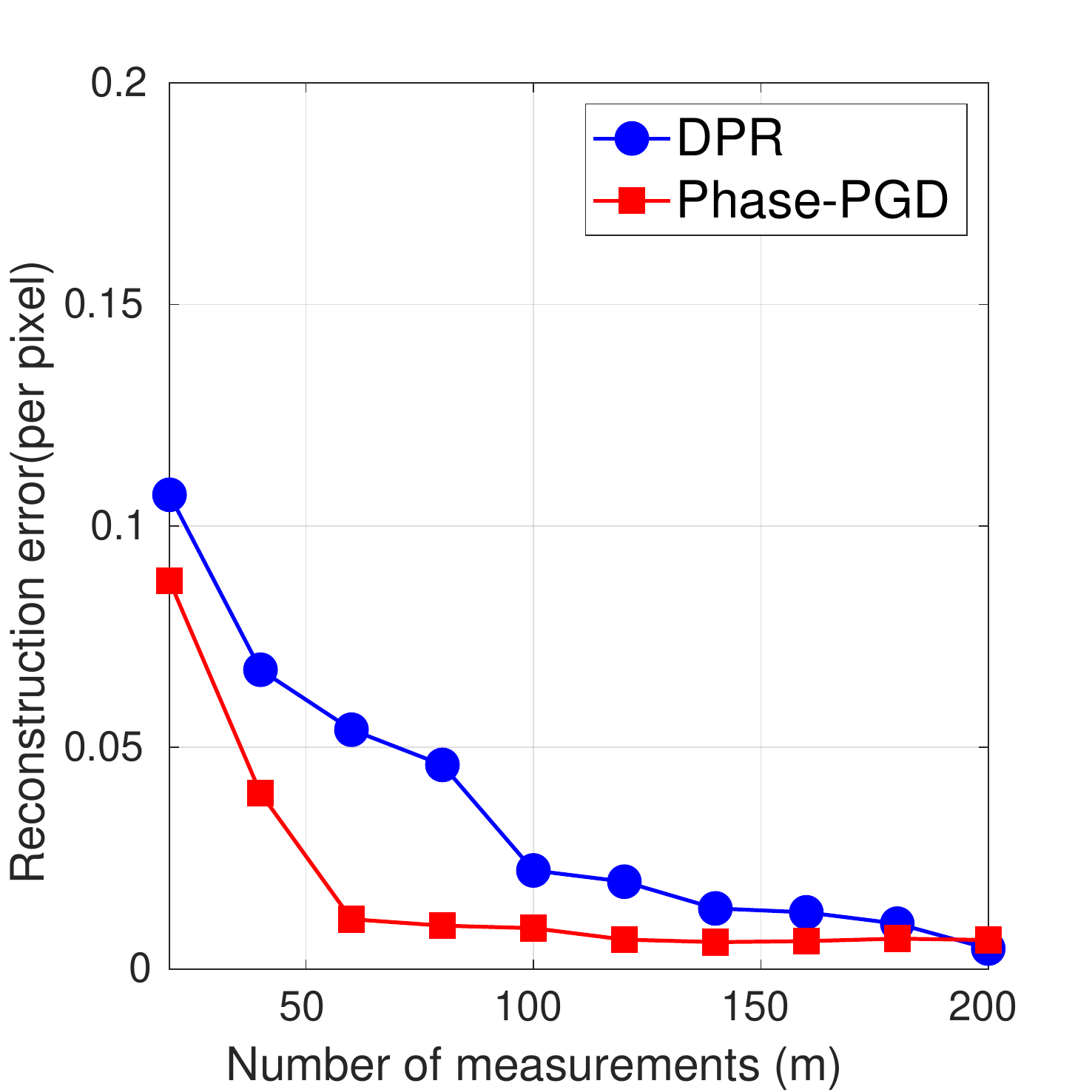}
		\caption{Reconstruction error}
		\label{fig:mnist-mse}
	\end{subfigure}\hfill%
	\begin{subfigure}[t]{0.24\textwidth}
		\includegraphics[width=\textwidth]{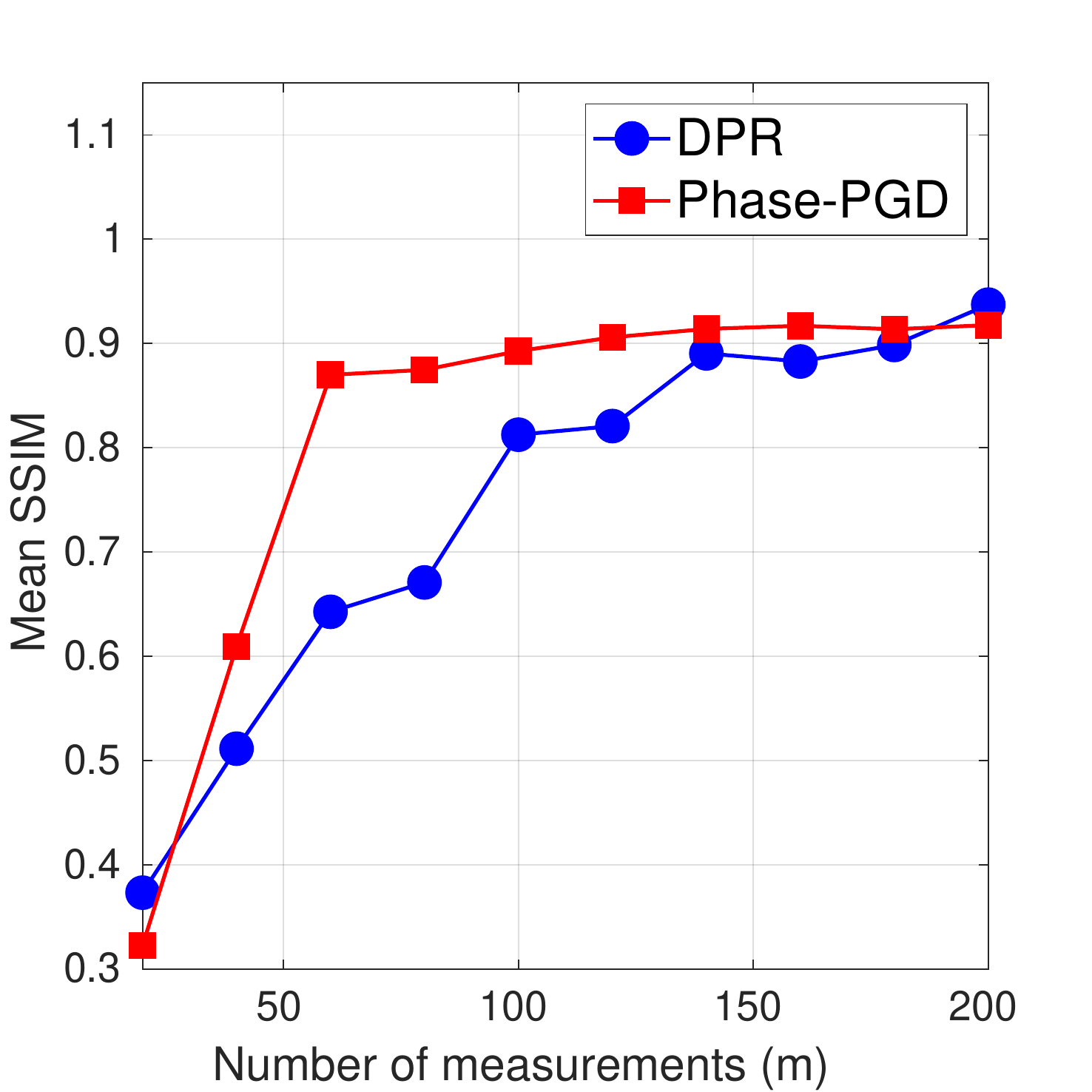}
		\caption{Mean SSIM}
		\label{fig:mnist-ssim}
	\end{subfigure}
	\caption{\textit{Comparison of Phase-PGD (ours) and DPR~\cite{hand2018phase} on MNIST test set.}}
\end{figure}

\begin{figure}
	\begin{subfigure}[t]{0.5\textwidth}
		\includegraphics[width=\textwidth]{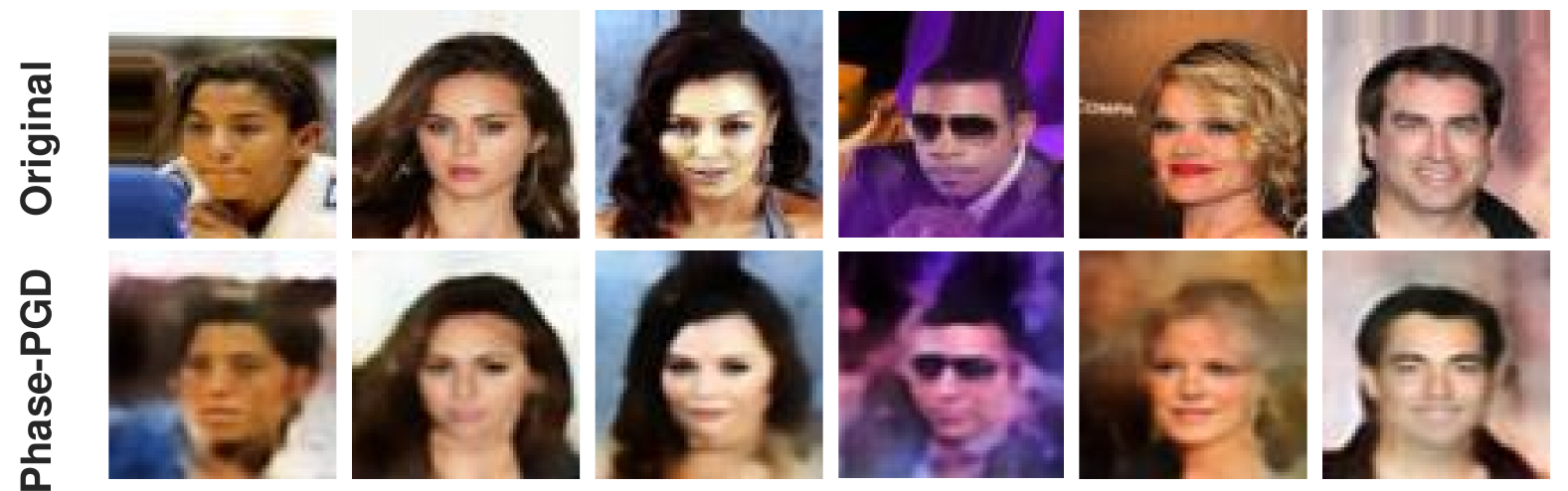}
		\caption{Reconstruction results on celebA dataset with $m = 1000$ measurements.}
		\label{fig:celeba-rec}
		\end{subfigure}\hfill%
   	\begin{subfigure}[t]{0.24\textwidth}
		\includegraphics[width=\textwidth]{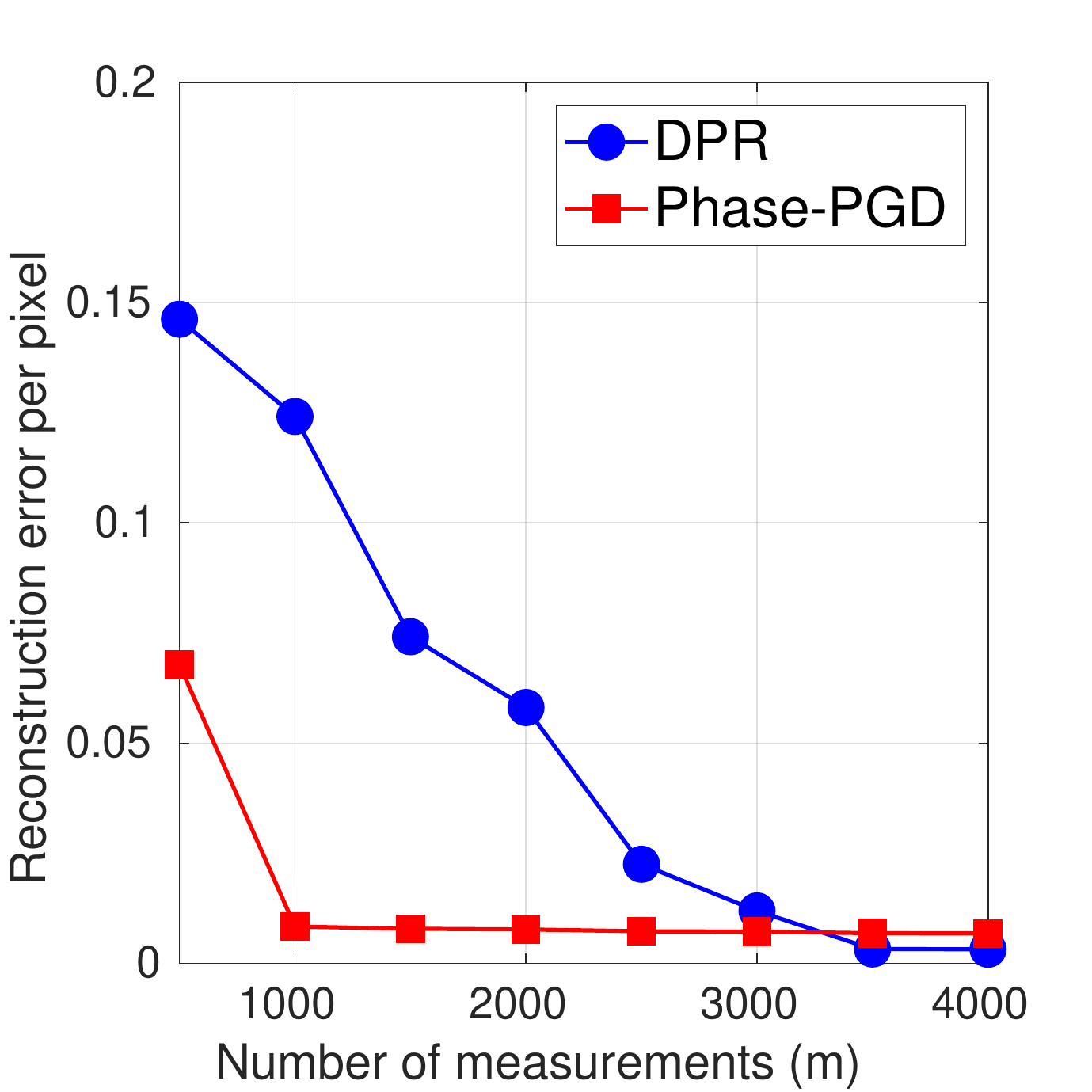}
		\caption{Reconstruction error}
		\label{fig:celeba-mse}
	\end{subfigure}\hfill%
	\begin{subfigure}[t]{0.24\textwidth}
		\includegraphics[width=\textwidth]{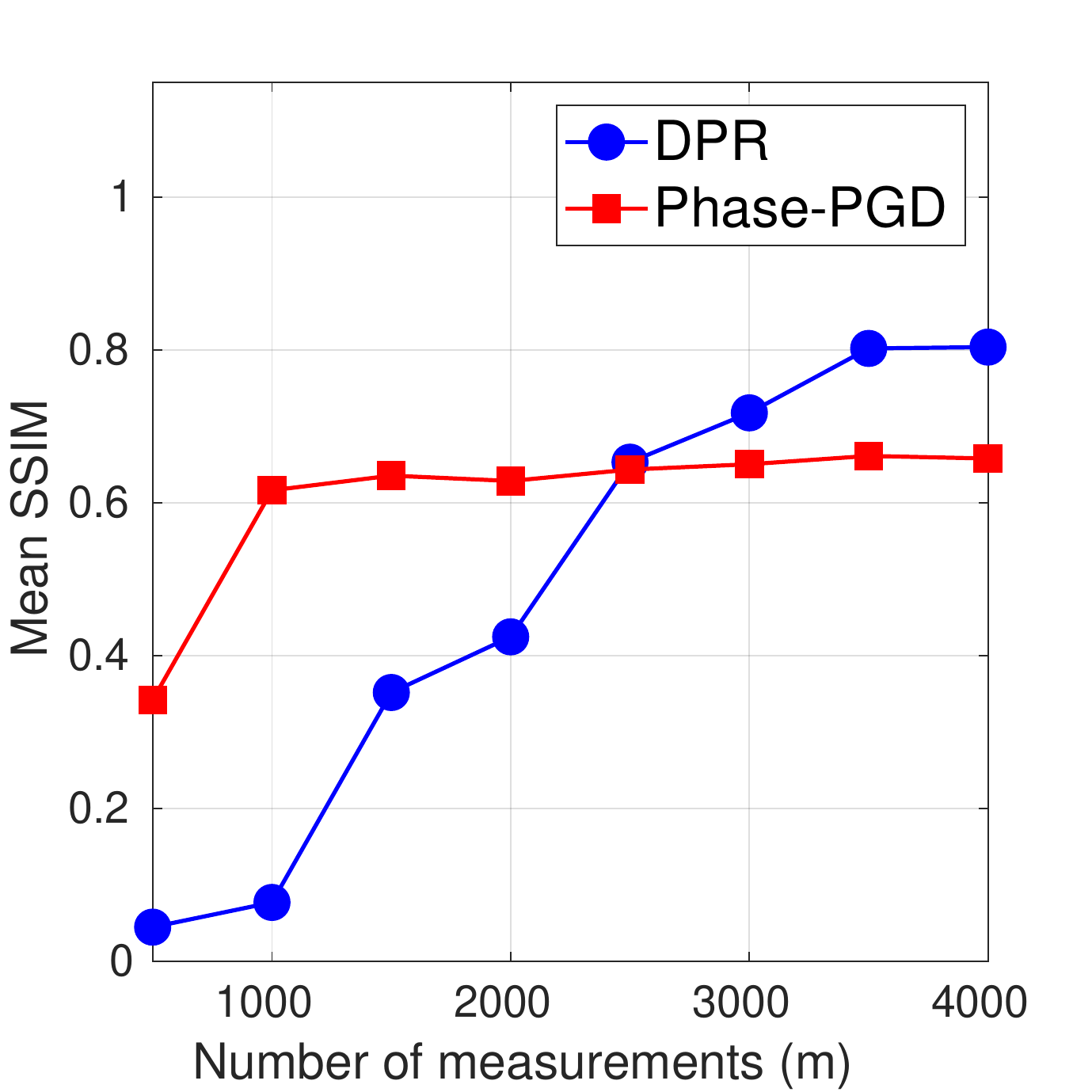}
		\caption{Mean SSIM}
		\label{fig:celeba-ssim}
	\end{subfigure}
	
	\caption{\emph{Comparison of Phase-PGD (ours) and DPR~\cite{hand2018phase} on celebA test set.}}
\end{figure}

\subsection{Experiments with phase retrieval}

In this section, we describe our experimental setup and report the performance comparisons of our proposed Phase-PGD and deep phase retrieval (DPR) method proposed in \cite{hand2018phase}. The DPR method estimates the signal as
\begin{equation}\label{eq:DPR}
     \widehat x = G\left(\argmin_z \; \| y - |\A G(\z)| \|_2^2\right), 
\end{equation}
which requires solving an optimization problem directly over the latent space $z$ of the generative model $G(\cdot)$. 

In our experiments, we choose the entries of the matrix $A$ independently from the $\mathcal{N}(0,\frac{1}{m})$ distribution. Although we ignore the presence of noise, it is possible to replicate our experiments with additive Gaussian noise. 

We use two different generative models for the MNIST and CelebA datasets. The generative model for CelebA follows the DCGAN framework \cite{radford2015unsupervised} except that we do not use any batchnorm layer since the gradient for this layer is dependent on batch size and the distribution of the batch. 
We train our generators by jointly optimizing generator parameters and the latent code $\z$ using SGD by following the procedure in  \cite{bojanowski2018optimizing}. We use the squared-loss function, $\|\x-\widehat{\x} \|^2 $ to train the generators. We choose $z$ from the standard normal distribution on $\mathbb{R}^k$ and normalize it to unit norm. We project $\z$ back to the unit norm  after each gradient update. 


In our first set of experiments, we use a generator trained over the MNIST training dataset resized to $32\times32$ pixel. We test two approaches on 10 images from the test set of MNIST dataset and provide both quantitative and qualitative results. For Phase-PGD, we choose gradient descent step size $\eta=0.9$ and $T=50$. For fair comparison, we use 2500 iterations for DPR. We show the reconstruction error comparison in Fig. \ref{fig:mnist-mse} and SSIM comparison in Fig.~\ref{fig:mnist-ssim} for increasing number of measurements. Since the input images are not chosen from the span of the generator, the reconstruction error does not necessarily reduce to zero. Nevertheless, we observe in Fig.~\ref{fig:mnist-mse} that Phase-PGD error gets close to zero with nearly 60 measurements, which is significantly smaller than those of DPR method.  Fig.~\ref{fig:mnist-rec} depicts reconstruction results for some of the selected MNIST images for the two approaches.

For our second set of experiments, we train a generator for the CelebA dataset. For training, we resize 202,599 color images of celebrity faces in celebA dataset to $64\times64\times3$ and kept $\frac{1}{32}$ of the images for testing. We do not use the aligned and cropped version, which includes only the faces in the images. 

We experiment on a subset of 10 images from the test dataset and report reconstruction results. We set the total number of updates to $1500$ with $T = 50$ for Phase-PGD. Image reconstruction results from $m=1000$ measurements using Phase-PGD algorithm are presented in Fig.~\ref{fig:celeba-rec}. A comparison between our method and DPR in terms of reconstruction error and SSIM is shown in Fig.~\ref{fig:celeba-mse} and Fig.~\ref{fig:celeba-ssim}, respectively.  We observe that Phase-PGD can achieve good reconstruction with significantly fewer measurements compared to DPR method. 